\documentclass[10pt,preprint]{article} 
\usepackage{tmlr}

\usepackage{amsmath,amsfonts,bm}

\def\eqref#1{equation~\ref{#1}}

\def\1{\bm{1}}

\DeclareMathAlphabet{\mathsfit}{\encodingdefault}{\sfdefault}{m}{sl}
\SetMathAlphabet{\mathsfit}{bold}{\encodingdefault}{\sfdefault}{bx}{n}

\newcommand{\R}{\mathbb{R}}

\usepackage{amssymb}
\usepackage{hyperref}
\usepackage{url}
\usepackage{graphicx}
\usepackage{algorithm}
\let\classAND\AND
\let\AND\relax
\usepackage{algorithmic}

\let\AND\classAND
	\AtBeginEnvironment{algorithmic}{\let\AND\algoAND}
\newtheorem{theorem}{Theorem}
\newtheorem{lemma}{Lemma}
\newtheorem{proposition}{Proposition}

\title{DFORD: Directional Feedback based Online Ordinal\\ Regression Learning}

\author{\name Naresh Manwani \email naresh.manwani@iiit.ac.in\\
      \addr Machine Learning Lab, Kohli Center on Intelligent Systems,
      IIIT Hyderabad, India
      \AND
      \name M Elamparithy \email elamparithy.m@research.iiit.ac.in \\
      \addr IIIT Hyderabad, India
      \AND
      \name Tanish Taneja \email tanish.taneja@research.iiit.ac.in\\
      \addr Machine Learning Lab, Kohli Center on Intelligent Systems,
      IIIT Hyderabad, India}

\def \X {\mathcal{X}}
\def \Y {\mathcal{Y}}
\def \xx {\mathbf{x}}
\def \I {\mathbb{I}}
\def \R {\mathbb{R}}

\def \ww {\mathbf{w}}
\def \uu {\mathbf{u}}

\def \thetaa  {\mbox{\boldmath $\theta$}}

\def \I {\mathbb{I}}
\def \zero {\mathbf{0}}

\def \Ex {\mathbb{E}}
\def \yt {\tilde{y}}
\def \H {\mathcal{H}}

\def \taut {\tilde{\tau}}

\newenvironment{proof}{\paragraph{Proof:}}{\hfill$\square$}
\def\month{MM}  
\def\year{YYYY} 
\def\openreview{\url{https://openreview.net/forum?id=XXXX}} 

\begin{document}

\maketitle

\begin{abstract}
In this paper, we introduce directional feedback in the ordinal regression setting, in which the learner receives feedback on whether the predicted label is on the left or the right side of the actual label. This is a weak supervision setting for ordinal regression compared to the full information setting, where the learner can access the labels. We propose an online algorithm for ordinal regression using directional feedback. The proposed algorithm uses an exploration-exploitation scheme to learn from directional feedback efficiently. Furthermore, we introduce its kernel-based variant to learn non-linear ordinal regression models in an online setting. We use a truncation trick to make the kernel implementation more memory efficient. The proposed algorithm maintains the ordering of the thresholds in the expected sense. Moreover, it achieves the expected regret of $\mathcal{O}(\log T)$. We compare our approach with a full information and a weakly supervised algorithm for ordinal regression on synthetic and real-world datasets. The proposed approach, which learns using directional feedback, performs comparably (sometimes better) to its full information counterpart.
\end{abstract}

\section{Introduction}

Ordinal Regression is a supervised learning problem that involves predicting the labels of data points that belong to ordered categories. For example, a movie rating/review system is based on an ordinal scale, usually from 1 to 5 stars, describing how good a movie is. In ordinal regression, there is an ordering relationship between the classes. Applications of ordinal regression include information retrieval \cite{10.1145/957013.957144}, collaborative filtering \cite{koren2013collaborative}, age estimation \cite{CAO2020325}, psychology \cite{ORPsych,doi:10.1177/2515245918823199}, healthcare \cite{doyle2014predicting}, and many more.

Although there are different techniques to model data with labels on an ordinal scale, there is a more significant issue of label availability and reliability, where many labels result from human feedback and preferences. 
For example, consider the Likert scale \cite{likert1932technique,likertSullivan}, which is frequently used in questionnaires and has categories that indicate the level of agreement or disagreement. A typical Likert scale will have categories like - \textit{strongly disagree, disagree, neither agree nor disagree, agree, strongly agree}. In situations like these, feedback on the labels might not be consistent from person to person. Even though the labels “agree” and “strongly agree” are different, it depends on an individual how they distinguish between these two categories. Ordinal labels, especially relating to human feedback, might not always be precise. 
For such situations, using a partial label can alleviate the ambiguous nature of human feedback, i.e., instead of an exact label, it is suggested that the feedback is less than or greater than a certain ordinal level. We call this type of partial label {\bf directional feedback}. To illustrate the importance of directional feedback, we take a look at the Hospital Anxiety and Depression Scale (HADS) \cite{https://doi.org/10.1111/j.1600-0447.1983.tb09716.x} where the scores (ranging between 0 to 21) obtained from a questionnaire are divided into an ordinal scale: Normal '$(<7)$’; ‘Borderline $(8-10)$’ and ‘Clinical depression (or anxiety) $(11+ )$.
Due to their relevance in real-life situations, data with such partial labels exists in larger quantities than data with the exact label. These types of semi-supervised settings are more prevalent. 

This paper proposes a new learning approach for ordinal regression, {\bf learning from directional feedback}. Note that this setting involves label uncertainty, as we only know whether the label predicted by the learner is below the actual label. To learn from directional feedback, we use an exploration-exploitation scheme to handle the label uncertainty in this paper. The main contributions of this paper are as follows.
\begin{itemize}
\item We propose a new weak supervision setting for ordinal regression called {\bf learning from directional feedback}.
    \item We propose an online algorithm for learning from the directional feedback, which we call {\bf DFORD}. We propose both linear and kernel versions of the DFORD algorithm. DFORD algorithm maintains the ordering of thresholds in the expected sense after every trial. 
    \item We show that the DFORD achieved expected regret of $\mathcal{O}(\ln T)$. 
    \item We provide extensive experimental results to show the effectiveness of the proposed approach and compare it with a full information-based baseline approach, as well as with another weak supervision-based approach.
\end{itemize}

\section{Related Work}

\subsection{Ordinal Regression}
There are two broad categories of approaches for learning ordinal regression. One way is to decompose the labels into several binary pairs and model them using classifiers. The second approach, known as threshold modelling, learns a set of thresholds closest to accurate ordering.
PRank, \cite{Crammer:2001} is an online ordinal regression algorithm that achieves a mistake bound in the order of $O(\sqrt{T})$.
\cite{Herbrich1999SupportVL} and \cite{Chu:2005} use a support vector learning algorithm to model the problem using binary decomposition and multi-thresholding models, respectively.
\cite{4633963,CAO2020325} uses a neural network to learn a function that maps the input vector to a probability vector, which estimates the number of categories the input vector belongs to. \cite{Garg2019RobustDO} proposes a label noise robust approach for deep ordinal regression.
\cite{article} proposes a projection algorithm combined with a regression approach. The final classifier combines pairwise class distances (PCD) and a regressor.
\cite{6548013} uses an ordinal ensemble algorithm that decomposes the problem into more straightforward classification tasks.
A Gaussian processes-based approach for ordinal regression is proposed in \cite{Chu:2005:GPO:1046920.1088707}. 
\cite{unknown} transforms the ordinal regression problem into a structured classification task, which is solved using conditional random fields. 
\cite{LI2018294} proposes an incremental Bayesian approach to address scalability issues in linearly non-separable datasets.

\subsection{Ordinal Regression Using Interval Labels}
Interval labels are a weak supervision setting in the context of ordinal regression \cite{pmlr-v39-antoniuk14}. In this setting, for every example in the training set, an interval label of the form $[y_l,y_r]=\{y_l,y_l+1,\ldots,y_r\}$ (assuming $y_l\leq y_r$) is given instead of the actual label $y$. It is further assumed that the actual label $y$ lies in the interval label set $[y_l,y_r]$. A large margin-based batch algorithm is proposed in \cite{pmlr-v39-antoniuk14}. Online algorithms for ordinal regression using interval labels are proposed in \cite{10.1145/3297001.3297011,8867949}. These algorithms enjoy a regret bound of $\mathcal{O}(\sqrt{T})$ for $T$ rounds.
The directional feedback setting proposed in this paper is entirely different from the interval label setting.

\section{Preliminaries}
\label{sec:prelims}
\paragraph{\bf Ordinal Regression}
Let $\X \subseteq \R^d$ be the instance space. Let $\Y=[K]$ be the ordered label space. Modeling ordinal regression requires a function $f:\mathcal{X}\rightarrow \mathbb{R}$ and $(K-1)$ thresholds $\thetaa=[\theta_1\;\cdots\;\theta_{K-1}]^\top$. To maintain the class order, we must ensure $\theta_1 \leq \;\ldots \;\leq \theta_{K}$. We assume $\theta_K=\infty$. For an example $\xx$, ordinal regression predicts the class label $h:\X\rightarrow \Y$ as follows
\begin{align*}
h(\xx) & = 1+\sum_{k=1} ^{K} \I_{\{f(\xx)>\theta_k\}}= \min_{i\in [K]} \{i: f(\xx) - \theta_i \leq 0\}
\end{align*}
Thus, the classifier splits the real line into $K$ consecutive intervals using thresholds $\theta_1,\ldots,\theta_{K-1}$ and then decides the class label based on which interval corresponds to $f(\xx)$.
\paragraph{\bf Linear Ordinal Regression Model} A linear ordinal regression model assumes $f$ as a linear function of $\xx$. Thus, $f(\xx)=\ww\cdot\xx$,
where $\ww \in \R^d$. Let $\mathbf{u}=(\ww,\thetaa)$ denote the combined parameter vector for linear ordinal regression.
\paragraph{\bf Kernel Ordinal Regression}
Kernel ordinal regression learns a nonlinear function $f:\mathcal{X}\rightarrow \mathbb{R}$ and $(K-1)$ ordered thresholds $\theta_1\leq \ldots\leq\theta_{K-1}$. $f$ can be modelled using the kernel function in the following way. Let $\mathcal{H}$ be a reproducing kernel Hilbert space (RKHS) \cite{10.5555/559923}. Thus, there exist a kernel $k:\mathcal{X}\times \mathcal{X}\rightarrow \mathbb{R}$ and dot product $\langle .,.\rangle_{\mathcal{H}}$ such that
(a) $k(.,.)$ has reproducing property (i.e., $\langle f,k(\mathbf{x},.)\rangle=f(\xx)$ and (b) $\mathcal{H}$ is closure of the span of all $k(\xx,.)$ with $\xx\in {\mathcal{X}}$. Thus, all $f\in \mathcal{H}$ can be represented as $f=\sum_{\xx_i \in \mathcal{X}}\alpha_i k(\xx_,.) $. The norm induced by the inner product $\langle .,.\rangle_{\mathcal{H}}$ is $\Vert f\Vert_{\mathcal{H}}=\langle f,f\rangle_{\mathcal{H}}^{\frac{1}{2}}$.

\paragraph{\bf Loss Function for Ordinal Regression}
The discrepancy between the predicted and actual labels can be measured using the mean absolute error (MAE) loss described below. 
\begin{equation}
\label{eq:MAE_Loss}
L^{A} (f,\thetaa,\xx,y) = \sum_{i=1}^{y-1} \I_{\{f(\xx) < \theta_i\}} + \sum_{i=y} ^{K} \I_{\{f(\xx) \geq \theta_i\}}
\end{equation}
This loss function takes value $0$ whenever $\theta_{y-1}\leq f(\xx)\leq \theta_{y}$. However, this loss function is discontinuous.
A convex surrogate of this loss function is as follows:
\begin{align}
L^{H}(f,\thetaa,\xx,y) &= \sum_{i=1}^{K} [-z_i(f(\xx)-\theta_i)]_+ \label{surrogate_loss}
\end{align}
where $[a]_+=\max(0,a)$ and $H$ stands for the hinge. Note that this loss function implicitly forces the ordering constraints of the thresholds $\theta_i$s \cite{Chu:2005}. {\bf Labels} $z_i,\;i\in[K]$ are defined as follows.
\begin{eqnarray}
\label{eq:dummy_labels}
z_i = \mathbb{I}[i\in \{1,\ldots, y-1\}]-\mathbb{I}[ i\in \{y,\ldots, K\}]
\end{eqnarray}
Loss $L^{H}(f,\thetaa,\xx,y)$ becomes zero only when $z_i(f(\xx)-\theta_i) \geq 0,\;\forall i \in [K]$.
The regularized loss for an example $\xx$ is defined as follows.
\begin{align}
\label{eq:IMC-Reg}
    L^{H}_{reg}(f,\thetaa,\xx,y) &= \frac{\lambda\Vert f^2 + \thetaa^2\Vert}{2} +\sum_{i=1}^{K} [-z_i(f(\xx) - \theta_i)]_+ 
\end{align}
When $f$ is linear, then $\Vert f\Vert^2=\Vert \ww\Vert^2$. One can see that $L^{H}_{reg}(f,\thetaa,\xx,y)$ is a strongly convex function of $(f,\thetaa)$. 

\paragraph{\bf Online Learning Setting}
In this paper, we are interested in online learning, in which examples are presented to the algorithm one at a time. Thus, the algorithm produces a sequence of ordinal regression models $(f^1,\thetaa^1),\cdots,(f^T,\thetaa^T)$ in $T$ rounds. The objective here is to minimize the sum of losses incurred in each round, which is $\sum_{t=1}^TL^{H}_{Reg}(f^t,\thetaa^t,\xx^t,y^t)$. Regret of an online ordinal regression algorithm is defined as the difference between the cumulative loss of the online algorithm and cumulative loss achieved by an optimal batch algorithm, for examples, $(\xx^1,y^1),\cdots,(\xx^T,y^T)$. Thus, regret $\mathcal{R}_T$ is written as:
\begin{align}
    \label{eq:regret-full-info}
\mathcal{R}_T=\sum_{t=1}^TL_{Reg}^{H}(f^t,\thetaa^t,\xx^t,y^t)-\min_{f,\thetaa}\sum_{t=1}^TL_{reg}^{H}(f,\thetaa,\xx^t,y^t)
\end{align}
 \paragraph{\bf Interval Label For Ordinal Regression} In this setting, we don't know the exact labels of the examples. Instead, we know the interval label is $\{y_l,y_{l}+1,\ldots,y_r\}$. It is assumed that the actual label $y\in\{y_l,y_{l}+1,\ldots,y_r\}$.
Here, we use interval-insensitive MAE loss, which is described as follows:  
\begin{align*}
L^{IA}(f,\boldsymbol{\theta},\mathbf{x},y_l,y_r)=\sum_{i=1}^{y_l-1}\mathbb{I}[f(\mathbf{x})<\theta_i]+\sum_{i=y_r}^{K}\mathbb{I}[f(\mathbf{x})\geq \theta_i]
\end{align*}
Here, we do not consider any loss for the labels $\{y_l,y_l+1,\ldots,y_r-1\}$. Note that this loss gives equal importance to each label in the interval. A hinge loss-based surrogate for this loss is described as follows \cite{10.1145/3297001.3297011}.
\begin{align}
L^{IH}(f,\thetaa,\xx,y) &= \sum_{i=1}^{K} [-z_i(f(\xx)-\theta_i)]_+ \label{surrogate_loss-interval}
\end{align}
where $z_i=\mathbb{I}[i<y_l]
        -\mathbb{I}[y_r\leq i\leq K]$.


\section{Directional Feedback and Label Uncertainty}In the standard setting, we assume that we receive the actual label $y$ for each example. However, in many situations, we don't observe the actual label $y$. We only get to know the relative position of predicted label $\tilde{y}$ concerning the actual label $y$. In other words, we only get the signal whether $\tilde{y}<y$. If directional feedback $\tilde{y}<y$, then the actual label lies in the set $\{\tilde{y}+1,\ldots,K\}$. Similarly, if $\tilde{y}\geq y$, then the actual label lies in the set $\{1,\ldots,\hat{y}\}$. Thus, each directional feedback results in a candidate label set. The actual label can be one among the candidate labels. This is a kind of weak supervision setting for ordinal regression. The goal is to learn ordinal regression using directional feedback. 

\subsection{Efficient Exploration of Labels Under Directional Feedback}
Given that we only get directional feedback, we first need to decide the label for which we should ask directional feedback. Let the model outputs the label $\hat{y}=\min_{i\in [K]}\big{\{}i\;:\;f(\xx)-\theta_i<0 \big{\}}$, then asking directional feedback for $\hat{y}$ seems the best choice. However, in this way, we do not explore other labels different from $\hat{y}$. Motivated by the explore-exploit idea in adversarial multi-arm bandits, we assign nonzero probabilities for exploring all labels. While doing so, we ensure that $\hat{y}$ is assigned the highest probability, and it decreases linearly on both sides as we move away from $\hat{y}$. 

At round $t$, we use a mixture of two probability distributions $P_1^t$ and $P_2^t$ as the label distribution. Thus,
\begin{align}
\label{eq:sampling-prob}
P^t=(1-\gamma)P^t_1+\gamma P_2^t,
\end{align}
where $P^t_1(i)=\mathbb{I}[i=\hat{y}^t],i\in[K]$. Here, distribution $P_1^t$ assigns probability $1$ to $\hat{y}^t$ and zero to other labels. Thus, $P_1^t$ is used to exploit the best possible choice in hand. On the other hand, $P_2^t$ assigns nonzero probabilities to all the labels. Thus, $P_2^t$ acts as an exploration distribution. We give weight $(1-\gamma)$ to exploitation and weight $\gamma$ to exploration. For $P_2^t$, we use a distribution which assigns maximum probability to $\hat{y}^t$ and it symmetrically decreases on both sides of $\hat{y}^t$. For simplicity, in this paper, we use a $P_2^t$ which linearly decreases on both sides of $\hat{y}^t$ with the same slope. Such a $P_2^t$ is defined as follows.
\begin{equation}
\label{eq:exploration-single}
P_2^t(i)\;=\;\frac{1+d_{\max}^t-\lvert i-\hat y^t\rvert}{Z^t}\,,\qquad i\in[K],
\end{equation}
where $d_{\max}^t := \max\{\hat y^t,\; K-\hat y^t\}$ and
the $Z^t$ is the normalization factor found as
\begin{align}
Z^t
&= K+K\,d_{\max}^t - \sum_{i=1}^K \lvert i-\hat y^t\rvert =K+ K\,d_{\max}^t
   - \frac{(\hat y^t-1)\hat y^t}{2}
   - \frac{(K-\hat y^t)(K-\hat y^t+1)}{2}. \label{eq:Z-general}
\end{align}
Equivalently, in closed form,
\[
Z^t=
\begin{cases}
(2K+1)\hat y^t-(\hat y^t)^2-\dfrac{K}{2}(K-1), & \text{if } 2\hat y^t\ge K,\\[6pt]
\dfrac{K(K+1)}{2}-\hat y^t(\hat y^t-1), & \text{if } 2\hat y^t< K.
\end{cases}
\]
It can be verified that $P_2^t(i)\ge 0$ for all $i$ and $\sum_{i=1}^K P_2^t(i)=1$.
  For example, Figure \ref{fig:C2} illustrates the discrete probability distribution for $K$=7 for two cases: (a) $\hat{y} = 5$ and (b) $\hat{y} = 2$.
\begin{figure}[h]
\centering
    \includegraphics[width=8cm]{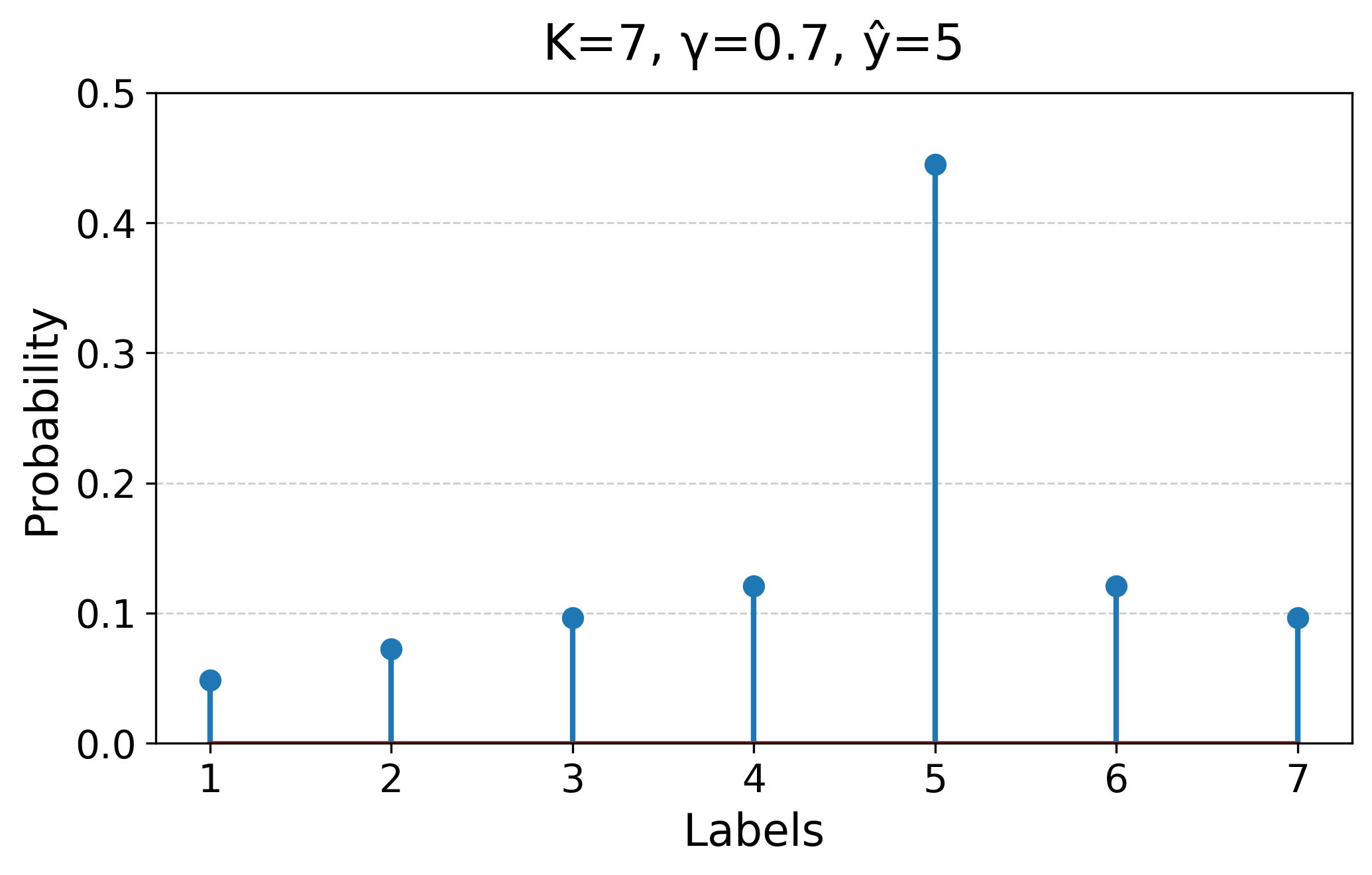}
    \includegraphics[width=8cm]{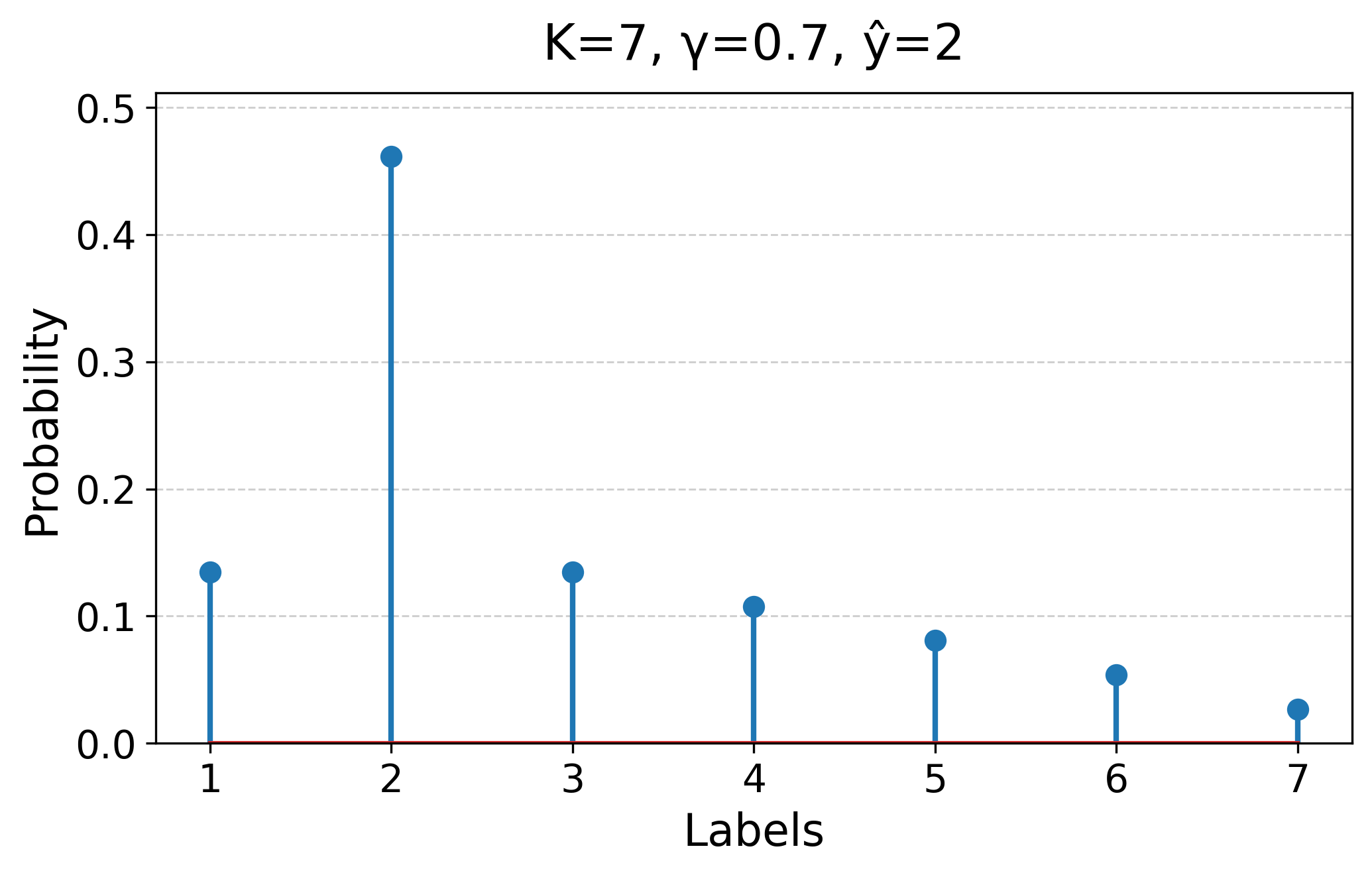}
    \caption{Example of label probability distribution for $K = 7$ and $\gamma=0.7$. (a) Case 1: $\hat{y} = 5$ and (b) Case 2: $\hat{y} = 2$. }
    \label{fig:C2}
\end{figure}

\subsection{Unbiased Estimators of Labels $z_i^t$} At round $t$, let $f^t$ and $\thetaa^t$ be the model parameters. Using these parameters, the model outputs label $\hat{y}^t=  \min_{i\in [K]} \{i: f^t(\xx^t) - \theta_i^t\leq 0\}$ for example $\xx^t$. However, we cannot compute the loss $L_{Reg}^{H}$ and its derivatives as we do not have the actual label $y^t$. Thus, we also do not have $z_i^t,\;i\in[K]$ (defined in eq.(\ref{eq:dummy_labels})).
Thus, we must define unbiased estimators $\tilde{z}_i^t$ of $z_i^t$ for every $i\in [K]$.
Now, the
algorithm samples a label $\tilde{y}^t$ using distribution $P^t(i)$. Algorithm predicts this label $\tilde{y}^t$ and receives the directional feedback $d_t=\mathbb{I}[\tilde{y}^t<y^t]$. Using $d_t$, we define $\tilde{z}_i^t,\;i\in [K]$ as follows. 
\begin{align}
  \label{eq:unbiased-estimator-z-1} \tilde{z}_i^t &=\frac{(2d_t-1)\mathbb{I}[i=\tilde{y}^t]}{P^t(i)}=\begin{cases}
        \frac{1}{P^t(\tilde{y}^t)}(2d_t-1),&i=\yt^t\\
        0,& i\neq \yt^t
    \end{cases}
\end{align}
Taking expectation of $\tilde{z}_i^t$ with respect to $P^t$, we get
    \begin{align*}
   & \mathbb{E}_{P^t}[\tilde{z}_i^t]=\mathbb{E}_{P^t}\left[\frac{1}{P^t(i)}\left(d_t\mathbb{I}[i=\tilde{y}^t]-(1-d_t)\mathbb{I}[i=\tilde{y}^t]\right)\right]\\
    &\quad\quad\quad=\sum_{j=1}^{K}\frac{P^t(j)}{P^t(i)}\mathbb{I}[j<y^t]\mathbb{I}[i= j]-\sum_{j=1}^{K}\frac{P^t(j)}{P^t(i)}\mathbb{I}[j\geq y^t]\mathbb{I}[i= j]\\
    &\quad\quad\quad=\mathbb{I}[i<y^t]-\mathbb{I}[i \geq y^t]=z_i^t.
    \end{align*}
Thus, $\tilde{z}_i^t$ is an unbiased estimator of $z_i^t$. 
\paragraph{\bf Unbiased Estimator of $\tau_i^t=z_i^t \I[z_i^t(\ww^t.\xx^t - \theta_i^t)<0]$:} Gradient of $L_{Reg}^{H}$ requires to know $\tau_i^t$. However, $\tau_i^t$ is a function of $z_i^t$ but can only access its unbiased estimator $\Tilde{z}_i^t$. Thus, we define an estimator $\taut_i^t$ for $\tau_i^t$ using $\tilde{z}_i^t$ as follows.
\begin{align}
\label{eq:unbiased-estimator-tau-1}
    \tilde{\tau}_i^t &= \tilde{z}_i^t\mathbb{I}[\tilde{z}_i^t(f^t(\xx^t) -\theta_i^t)\leq 0]
\end{align} 
\begin{theorem}\label{ub-est}
$\tilde{\tau}_i^t$ is an unbiased estimator of $\tau_i^t$. Thus, for $i=1\ldots K$, we have $\mathbb{E}[\tilde{\tau}_i^t]=\tau_i^t$.
\label{unbiased}
\end{theorem}
\section{DFORD-Linear: Online Linear Ordinal Regression Using Directional Feedbacks}
Here, we present the DFORD algorithm to learn the linear ordinal regression model. Which means, $f(\xx)=\ww\cdot \xx$. Thus, the parameters to be estimated are $\ww$ and $\thetaa$. We initialize with $\ww^2=\mathbf{0}$ and $\thetaa^2=\mathbf{0}$. Let $\ww^t,\thetaa^t$ be the estimates of the parameters at the beginning of trial $t$. Let $t$, $\xx^t$ be the example observed. We find $\hat{y}^t=  \min_{i\in [K]} \{i: \ww^t.\xx^t - \theta_i^t\leq 0\}$. We define $P^t$ using eq.(\ref{eq:sampling-prob}). We now sample $\yt^t$ using $P^t$ and predict that. We receive directional feedback $d_t=\mathbb{I}[\yt^t<y^t]$. Using $d_t$, we define $\tilde{z}_i^t$. Since we only receive directional feedback, finding a gradient of $L_{reg}^{H}$ is impossible. Instead, we use an unbiased estimator of the gradient as follows- 
\begin{equation}
g(\mathbf{w}^t,\thetaa^t,\xx^t,d^t) = \begin{pmatrix}\lambda \ww^t - \sum_{i=1} ^{K}\tilde{\tau}_i^t \xx^t \quad  \lambda \theta_1^t + \tilde{\tau}_1^t \quad \cdots \quad   \lambda \theta_{K-1}^t + \tilde{\tau}_{K-1}^t\end{pmatrix}
\label{eqn:pegasos1}
\end{equation}
It can be easily seen that $\mathbb{E}_{\tilde{y}\sim P^t}[g(\mathbf{w}^t,\thetaa^t,\xx^t,d^t)]=\nabla_{\uu} L_{reg}^{H}(\ww^t,\thetaa^t,\xx^t,y^t)$. Thus, we use $g(\mathbf{w}^t,\thetaa^t,\xx^t,d^t)$ as a proxy to $\nabla_{\uu} L_{reg}^{H}(\ww^t,\thetaa^t,\xx^t,y^t)$\footnote{Note that for a linear ordinal regression model, $f$ depends on $\ww$. Thus, for linear case, loss for $\xx^t$ using parameters $(\ww^t,\thetaa^t)$ is expressed as $L_{reg}^{H}(\ww^t,\thetaa^t,\xx^t,y^t)$.} in the update equation. 
\begin{algorithm}[t!]
\caption{DFORD-Linear: Directional Feedback Based Online Linear Ordinal Regression}
\label{algo5}
\begin{algorithmic}[1]
\STATE {\bf Input: } Training Dataset $\mathcal{S}$, $\lambda$\;
\STATE {\bf Initialize} Set $t=2$, $\ww^2=\zero$, $\theta_1^{2}=\ldots=\theta_{K-1}^2=0$, $\theta_K^2=\infty$\;
\FOR{$i\leftarrow 2$ to $T$}
\STATE Get example $\xx^t$ 
\STATE Find $\hat{y}^t$ as $\hat{y}^t = \min_{i\in [K]} \{i: \ww^t\cdot\xx^t - \theta_i^t \leq 0\}$
\STATE Define sampling probability distribution $P^t(r),\;r=1\ldots K$ using eq.(\ref{eq:sampling-prob})
\STATE Sample $\tilde{y}^t$ from distribution $P^t$. Predict $\tilde{y}^t$
\STATE Observe the directional feedback $\mathbb{I}[\tilde{y}^t<y^t]$.
\STATE $\tilde{z}_i^t=\frac{1}{P(i)}(2d^t-1)\mathbb{I}[i = \tilde{y}^t],\;i=1\ldots K$
\STATE Initialize $\tilde{\tau}_i^t=0,\;i\in [K]$
\FOR{$i\in [K-1]$}
\IF{$\tilde{z}_i^t(\ww^t.\xx^t -\theta_i^t)\leq 0$}
\STATE $\tilde{\tau}_i^t = \tilde{z}_i^t$
\ENDIF
\ENDFOR
\STATE $\eta_t =  \frac{1}{\lambda t}$
\STATE $\ww^{t+1} = (1-\eta_t \lambda)\ww^t + \eta_t\tilde{\tau}_{\yt^t}^t \xx^t$\;
\STATE $\theta_i^{t+1}=(1-\eta_t\lambda)\theta^t_i- \mathbb{I}[i=\yt^t]\eta_t\taut_{\yt^t}^t,\;i=1\ldots K$
\ENDFOR
\STATE {\bf Output}: $h(\xx)=\min_{i\in [K]}\big{\{}i\;:\;\ww^{T+1}.\xx-\theta_i^{T+1} <0 \big{\}}$
\end{algorithmic}
\end{algorithm}
 We update the parameters as follows.
\begin{align*}
    \ww^{t+1}
    &=(1-\eta_t\lambda)\ww^t + \eta_t\xx^t\taut^t_{\yt^t}\\
\theta_i^{t+1}
    &=(1-\eta_t\lambda)\theta^t_i- \mathbb{I}[i=\yt^t]\eta_t\taut_{\yt^t}^t;\;\;i=1\ldots K-1
\end{align*}
We repeat this process for $T$ rounds. Complete details of the DFORD-Linear approach are given in Algorithm~\ref{algo5}.
\begin{lemma}{\bf Order Preservation of DFORD-Linear:}
\label{lemma:pril1}
Let $\ww^t$ and $\theta_1^t,\ldots,\theta^t_K$ be the parameters at the beginning of round $t$ of the DFORD-Linear Algorithm. For $t\geq 2$, let $\theta_1^t,\ldots,\theta^t_K$ be such that $\mathbb{E}[\theta_{i+1}^t-\theta_i^t]\geq \frac{K\eta_t}{1-\eta_t\lambda},\;\forall i\in [K-1]$, then DFORD-Linear Algorithm ensures that $\mathbb{E}[\theta_{i+1}^{t+1}-\theta_i^{t+1}]\geq 0,\;\forall i\in [K-1],\;t\geq 2$
\end{lemma}
The above lemma shows that the threshold ordering is preserved in the expected sense. This happens because our algorithms make probabilistic updates in each iteration.
\paragraph{\bf Regret Bound for DFORD-Linear}
In the proposed algorithm DFORD-Linear, at any round $t$, the model parameters $\uu^t=(\ww^t,\thetaa^t)$ depend on the random variables $\yt^1,\ldots,yt^{t-1}$. Thus, the regret defined in eq.(\ref{eq:regret-full-info}) also becomes a function of random variables $\yt^1,\ldots,\yt^T$. Thus, for DFORD-Linear, we find expected regret as follows.
\begin{align*}
    \Ex_{P^{[T]}}[\mathcal{R}_T]&=\Ex_{P^{[T]}}\left[\sum_{t=1}^TL_{Reg}^{H}(\uu^t,\xx^t,y^t)\right]-\min_{\uu}\sum_{t=1}^TL_{Reg}^{H}(\uu,\xx^t,y^t)
\end{align*}
where $P^{[T]}=P(\yt^1,\ldots,\yt^T)$ is the joint distribution of $\yt^1,\ldots,\yt^T$.
\begin{theorem}
    Let $\xx^1,\ldots,\xx^T$ be the sequence of examples presented to DFORD-Linear. Let $(\uu^1), (\uu^2), \hdots, (\uu^{T+1})$  be the sequence of parameter vectors generated by the DFORD-Linear Algorithm, where $\uu^t=(\ww^t,\thetaa^t)$. Let $\|\xx^t\| \leq R,\;\forall t\in [T]$. Then, for any $(\ww, \thetaa)$, we have,
    \begin{align*}
        &\Ex_{P^{[T]}}\left[\sum_{t=1}^TL_{Reg}^{H}(\uu^t,\xx^t,y^t)-\sum_{t=1}^TL_{Reg}^{H}(\uu,\xx^t,y^t)\right] \leq \frac{16K^2(R^2 + 1)\ln K\ln{T}}{\lambda\gamma} .
    \end{align*}
\end{theorem}
The above results show that DFORD-Linear achieves $\mathcal{O}(\log T)$ regret in the expected sense. The regret bound depends on $K$ (number of classes) as $K^2(1+\ln K)$.

\section{Kernel DFORD}
Kernel DFORD works like the linear case except for a few changes. In the case of Kernel DFORD, we use the following unbiased estimator of $\nabla_{f,\thetaa}L_{Reg}^{H}(f^t,\thetaa^t,\xx^t,y^t)$.
\begin{align*}
    \mathbf{g}(f^t,\thetaa^t,\xx^t,d^t)
    =\begin{bmatrix}
        \lambda f^t-\taut_{\yt^t}^t k(\xx^t,.)\\
    \lambda \thetaa^t + \taut^t_{\yt^t}\mathbf{e}_{\yt^t}
    \end{bmatrix}
\end{align*}
where $\mathbf{e}_{\yt^t}$ is a $(K-1)$-dimensional column vector whose all elements are zero except for the elements indexed at $\yt^t$, which is 1.
\begin{algorithm}[t!]
\caption{DFORD-Kernel: Directional Feedback Based Online Kernel Ordinal Regression with truncation}
\label{algo6}
\begin{algorithmic}[1]
    \STATE {\bf Input: } Training Dataset  $\mathcal{S}$, $\lambda$, $\gamma$, $\delta$,  Kernel function $k(.,.)$, parameters of the kernel function $k(.,.)$
    \STATE {\bf Initialize:} Set $t = 2, f^2 = 0, \theta_1^2=0, \hdots, \theta_{K-1}^2=0$
    \FOR{$t \leftarrow 2$ to $T$}
   \STATE Get example $\xx^t$
    \STATE Find $\hat{y}^t = \min\; \{i\;|\;f^t(\xx^t)-\theta^{t}_i\leq 0\}$
    \STATE Define distribution $P^t(r), r = 1, \hdots K$ using eq.(\ref{eq:sampling-prob})
\STATE Sample $\tilde{y}^t \sim P^t$. Predict $\yt^t$. Observe the directional feedback $d_t = \mathbb{I}[\tilde{y}^t<y^t]$.
\STATE Define $\tilde{z}_i^t=\frac{(2d_t-1)}{P^(i)}\mathbb{I}[i = \tilde{y}^t],\;i=1\ldots K$
\STATE Initialize $\tilde{\tau}_i^t=0,\;i\in [K]$
\FOR{$i\in [K-1]$}
\IF{$\tilde{z}_i^t(f^t(x^t) -\theta_i^t)\leq 0$}
\STATE $\tilde{\tau}_i^t = \tilde{z}_i^t$
\ENDIF
\ENDFOR
\STATE $f^{t+1}=\frac{1}{\lambda t}\sum_{i=\max(1,t-\delta)}^t\taut^i_{\yt^i}k(\xx^i,.)$
\STATE $\theta_i^{t+1} = (1-\eta_t\lambda)\theta_i^t- \mathbb{I}[i=\yt^t]\eta_t\taut^t_{\yt^t};\;\;i=1\ldots K-1$
\ENDFOR
\end{algorithmic}
\end{algorithm}
Using $\mathbf{g}(f^t,\thetaa^t,\xx^t,d^t)$, 
the update equation for $f$ becomes
\begin{align*}
    f^{t+1}&=f^t -\eta_t [\lambda f^t-\taut_{\yt^t}^tk(\xx^t,.)]=(1-\eta_t\lambda)f^t + \eta_t\taut_{\yt^t}^tk(\xx^t,.)
\end{align*}
\begin{lemma}
    Let $\eta_t = \frac{1}{\lambda t}$, $f^1 = 0$. Then, the update rule for $f^t$ given in Algorithm 3 reduces to $
    f^{t+1}(.) = \frac{1}{\lambda t}\sum_{i = 1}^{t} \taut^i_{\yt^i} k(\xx^i,.)$.
\end{lemma}
To compute $f^{t+1}(\xx)$ we will need all $\xx^{i}, i=1, \hdots, t$. We also have to perform a linearly increasing number of kernel product computations between $\xx$ and $\xx^i$ in each iteration of $t$. This makes the kernel algorithm memory and compute-intensive. To tackle these issues, we use a technique called \textbf{\textit{truncation}} (\cite{kivinen2004online}) wherein we omit the $\xx^i$s from $i=1$ up to a certain $i = t - \delta$ ($\delta$ is called the truncation parameter) while computing $f^{t+1}$ as follows. 
\begin{align*}
    f^{t+1}(.) &= \frac{1}{\lambda t}\sum_{i = \max(1, t - \delta)}^{t} \taut^i_{\yt^i}k(\xx^i,.)
\end{align*}
This reduces the memory and computations required by a considerable amount while keeping the difference between the true $f^{t+1}$ and truncated $f^{t+1}$ close in value so that truncated $f^{t+1}$ can be used. Parameters $\thetaa$ are updated as follows.
\begin{align*}
\theta_i^{t+1}
&=\theta_i^{t+1} = \theta_i^t- \mathbb{I}[i=\yt^t]\eta_t\taut^t_{\yt^t};\;i=1\ldots K
\end{align*}
Complete details of DFORD-Kernel are given in Algorithm~\ref{algo6}. DFORD-Kernel also maintains the orders of thresholds after each iteration in the expected sense. The proof works in the same way as given in Lemma~4.2.

\begin{theorem}[Regret Bound for DFORD-Kernel]
Let the sequence $\xx^1, \ldots, \xx^T$ be a sequence of examples presented to the DFORD-Kernel approach in $T$ trials. Let $k(.,.)$ be the kernel function used such that $k(\xx^t, \xx^t) \leq R_1^2,\; \forall t \in [T]$. Let $(f^1,\thetaa^1),\ldots,(f^T,\thetaa^T)$ be the kernel classifiers generated by the DFORD-Kernel algorithm in $T$ trials. Then, for any $f \in \H$ and $\thetaa \in \mathbb{R}^{K-1}$, we have
    \label{reg_ker}
    \begin{align*}
&\Ex_{P^{[T]}}\left[\sum_{t=1}^T\left[L_{reg}^{H}((f^t, \thetaa^t),\xx^t, y^t) -L_{reg}^{H}((f, \thetaa),\xx^t, y^t)\right]\right]  \leq \frac{16}{\gamma\lambda}(R_1^2+1)K^2\ln K\ln T.
    \end{align*}
\end{theorem}

\subsection{Gradient Clipping}
We see that the gradient estimator $\mathbf{g}()$ for DFORD depends on $\taut_{\yt^t}^t$ and hence on $z_{\yt^t}=\frac{2d^t-1}{P^t(\yt^t)}$. The denominator term $P^t(\yt^t)$ can be very small, making the gradient explode. This might result in huge steps in each trial, eventually affecting the convergence. To avoid this issue, we use gradient clipping, which ensures that the norm of the gradient is less than or equal to some pre-specified value $\alpha >0$. Thus, we use the following estimator for the gradients.
\begin{align*}
    \mathbf{g}(f^t,\thetaa^t,\xx^t,d^t)\leftarrow \frac{\mathbf{g}(f^t,\thetaa^t,\xx^t,d^t)}{\Vert \mathbf{g}(f^t,\thetaa^t,\xx^t,d^t)\Vert}\alpha
\end{align*}

\section{Experiments}
\subsection{Datasets Used:} We use seven real-world datasets: Abalone, Parkinson's Telemonitoring \cite{UCI-Repo};  EachMovie \cite{eachmovie}, MSLR Web 10K \cite{DBLP:journals/corr/QinL13}, California Housing, Computer Activity, Bank \cite{torgo2005regression}. We also show experiments on a synthetic dataset proposed in \cite{Crammer:2001}. Out of the eight datasets, the target values of 5 datasets were discretized into ordinal categories. 
Target values for the Abalone dataset are discretized to an ordinal scale by converting interval $[1, 7]$ to $1$, $(7,9]$ to $2$, $(9, 12]$ to $3$ and 12, onwards to $4$.
The California Housing, Computer Activity, Parkinson's Telemonitoring, MSLR Web 10K and Bank datasets were normalized, and their target values have been divided into 10 equi-frequent categories as done in \cite{Chu:2005}.
Labels of the Parkinson's Telemonitoring dataset are discretized as multiples of 5.

\subsection{Baselines:}
We compare the proposed approach with two baselines. The pseudocode for both these algorithms can be found in the supplementary file.
\subsubsection*{\bf 1. PRank \cite{Crammer:2001} (Full Information Based Online Ordinal Regression: }
 In the full information setting, the algorithm has access to the actual label $y^t$ of example $\mathbf{x}^t$ at each round $t$. 
At round $t$, the algorithm performs stochastic gradient descent on $L^{H}_{reg}(f,\thetaa,\xx^t,y^t)$ (see, eq.(\ref{eq:IMC-Reg})) to update the parameters. 
The complete details of the approach PRank are given in the Supplementary file. 
\subsubsection*{\bf 2. PRIL \cite{10.1145/3297001.3297011}: Interval Label Based Online Ordinal Regression: }In case of interval labels, as discussed in Section~\ref{sec:prelims}, at trial $t$, we observe example $\mathbf{x}^t$ and interval label $\{y_l^t,y_l^{t}+1,\ldots,y_r^t\}$. For a linear model, PRIL \cite{10.1145/3297001.3297011} performs stochastic gradient descent on $L^{IH}(\mathbf{w},b,\mathbf{x}^t,y_l^t,y_r^t)$ (see, eq.(\ref{surrogate_loss-interval})).

Here, we discuss how PRIL can be used when we have bandit feedback. Consider learning for a linear classifier, which means, $f(\xx)=\ww\cdot\xx$. We initialize with $\ww^0=\mathbf{0}$ and $\thetaa^0=\mathbf{0}$. Let $\ww^t,\thetaa^t$ be the estimates of the parameters at the beginning of trial $t$. Let at trial $t$, $\xx^t$ be the example observed, and $\hat{y}^t = \min_{i\in [K]} \{i: \ww^t\cdot\xx^t - \theta_i^t \leq 0\}$ be the label predicted. Now, we get the directional feedback $d_t=\mathbb{I}[\hat{y}^t<y^t]$. If $d_t=1$, then we get to know that $\hat{y}^t<y^t$ or $y^t\in \{\hat{y}^t+1,\ldots,K\}$. Thus, in case $d_t=1$, we set $y_l^t=\hat{y}^t+1$ and $y_r^t=K$. On the other hand, if $d_t=0$, then we find that $\hat{y}^t\geq y^t$ or $y^t\in \{1,\ldots, \hat{y}^t\}$. Thus, in case $d_t=0$, we set $y_l^t=1$ and $y_r^t=\hat{y}^t$. We define $z_i^t,\;i\in [K]$ as follows:
\begin{align*}
    z_i^{t}&=\begin{cases}
        \mathbb{I}\left[i\in\{1,\ldots,\hat{y}^t\}\right], & \text{if }\hat{y}^t<y^t\\
-\mathbb{I}[i\in \{\hat{y}^t,\ldots,K-1\}], & \text{if }\hat{y}^t\geq y^t
\end{cases}
\end{align*}
At round $t$, the algorithm performs stochastic gradient descent on $L^{IH}_{reg}(f,\thetaa,\xx^t,y_l^t,y_r^t)$ to update the parameters. 
$\ww^{t+1}$ and $\thetaa^{t+1}$ are found as follows.
\begin{align*}
\ww^{t+1} &= \ww^t - \eta_t\nabla_{\ww} L^{IH}_{reg}(\ww^t.\xx^t,\thetaa^t,\hat{y}^t+1,K)\mathbb{I}[\hat{y}^t<y^t]  - \eta_t\nabla_{\ww} L^{IH}_{reg}(\ww^t.\xx^t,\thetaa^t,1,\hat{y}^t)\mathbb{I}[\hat{y}^t\geq y^t]\\
&= (1-\eta_t\lambda)\ww^t + \eta_t\mathbb{I}[\hat{y}^t<y^t]\sum_{i=1}^{\hat{y}^t}\tau_i^t \xx^t - \eta_t\mathbb{I}[\hat{y}^t\geq y^t]\sum_{i=\hat{y}^t}^{K-1}\tau_i^t \xx^t\\ 
\theta^{t+1}_i &= \theta^t_i - \eta_t\frac{\partial  L^{H}_{reg}(\ww^t.\xx^t,\thetaa^t,\xx^t,y^t)}{\partial \theta_i}\\
&=\begin{cases}
    (1-\eta_t\lambda)\theta_i^t - \eta_t \tau_i^t\mathbb{I}[i\in \{1,\ldots,\hat{y}^t\}],\; \text{if }\hat{y}^t<y^t\\
    (1-\eta_t\lambda)\theta_i^t + \eta_t \tau_i^t\mathbb{I}[i\in \{\hat{y}^t,\ldots,K-1\}],\; \text{if }\hat{y}^t\geq y^t
    \end{cases}
\end{align*}
where $\tau_i^t=z_i^t \I[z_i^t(\ww^t.\xx^t - \theta_i^t)\leq 0].$

\subsection{Comparison Metrics: } We compare our algorithm with the baselines using average loss over $T$ iterations and the constraint violations. The loss recorded here is the Mean Absolute Error (MAE) (eq. (\ref{eq:MAE_Loss})). As the proposed approach guarantees the ordering of thresholds in the expected sense, we also monitor the number of constraint violations among thresholds. We record the average number of violations in $T$ trials. We do average over 10 independent runs.

\subsection{Choosing $\lambda$ and $\gamma$ for DFORD:} 
For selecting $\lambda$ and $\gamma$ for DFORD, we perform a grid search for each dataset. Our algorithm depends on the regularization parameter Lambda ($\lambda$) and exploration-exploitation parameter Gamma ($\gamma$). 
\begin{figure}[H]
    \centering
    \includegraphics[width=0.44\columnwidth]{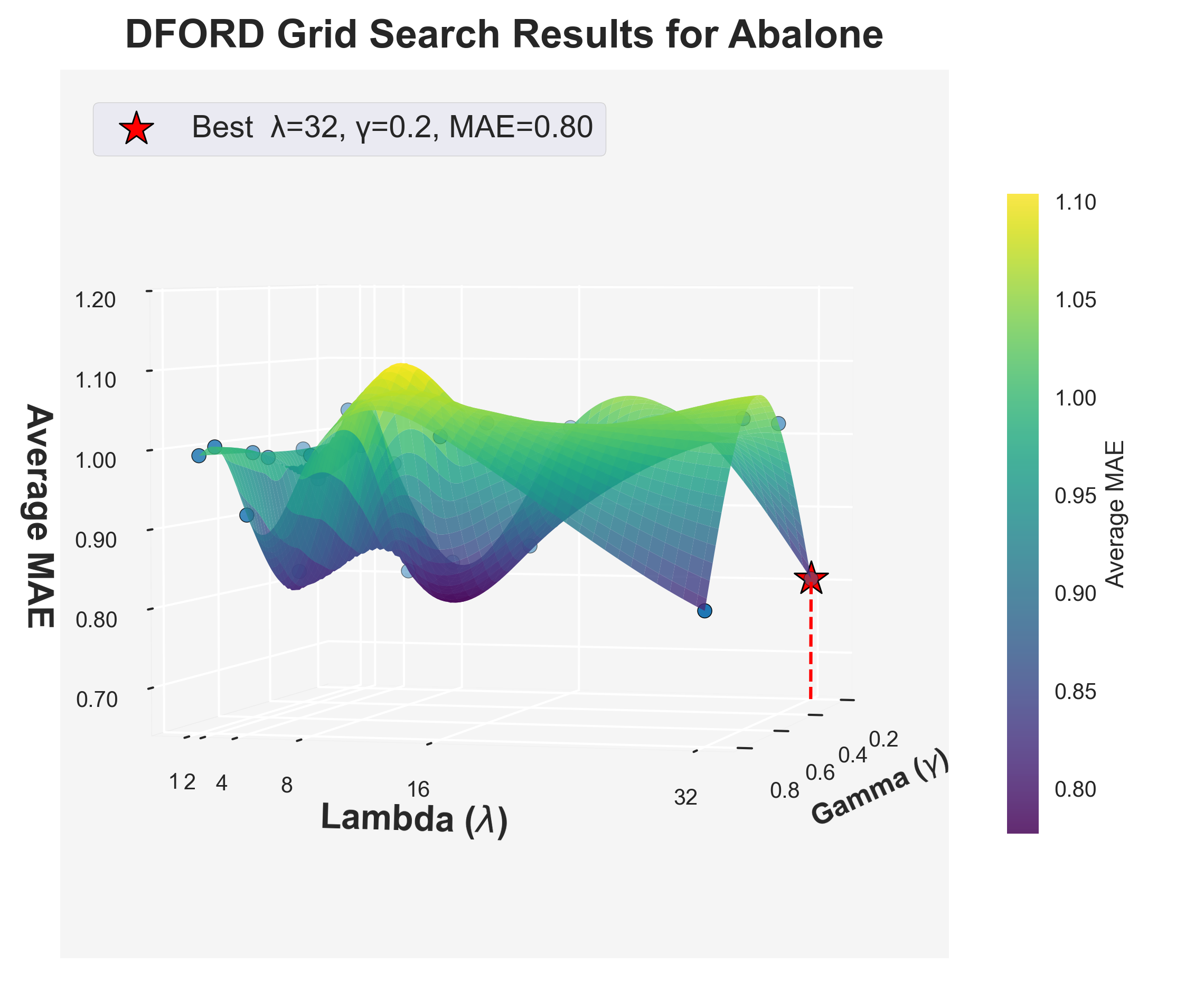}
    \includegraphics[width=0.44\columnwidth]{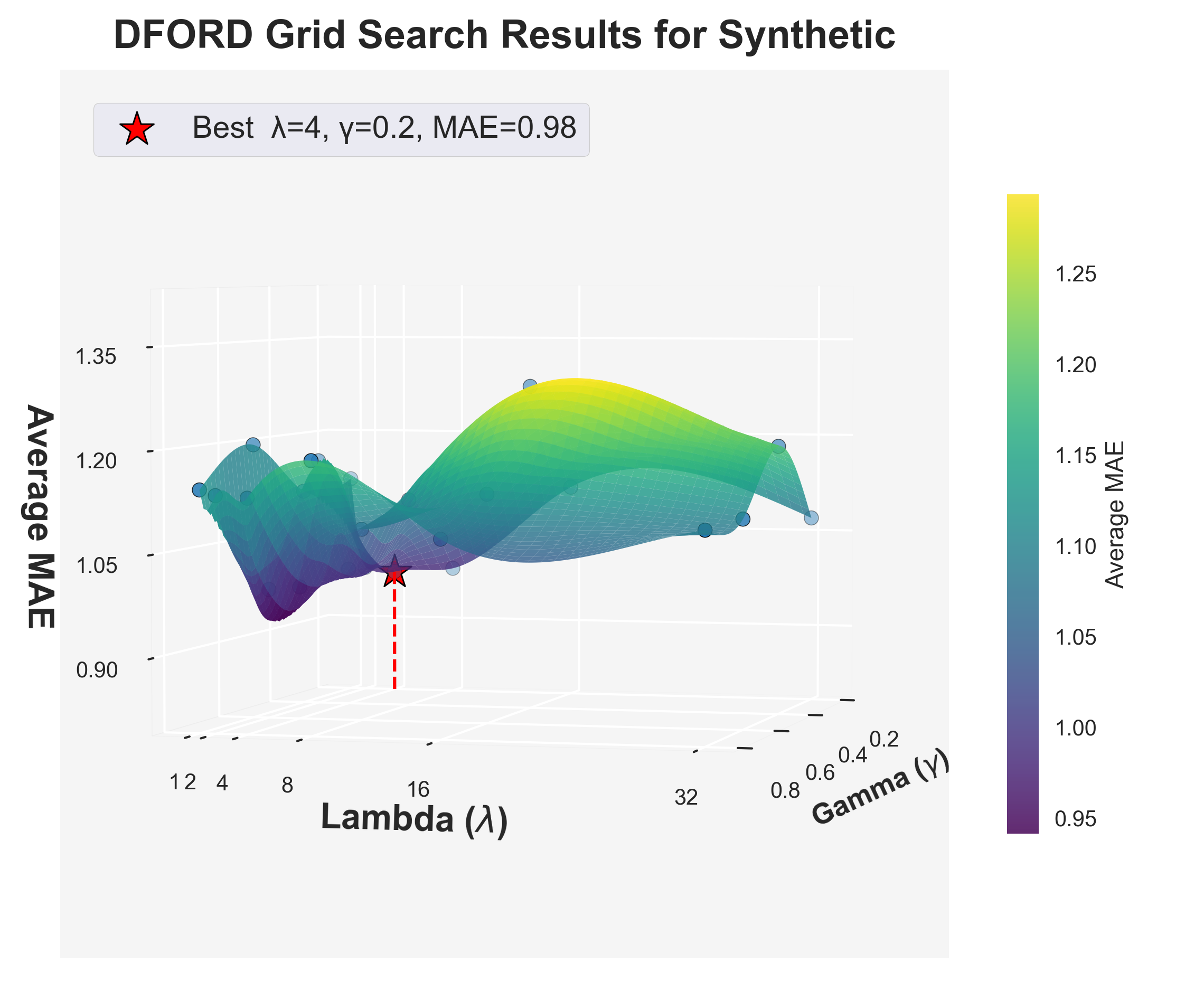}
    \\
    \includegraphics[width=0.44\columnwidth]{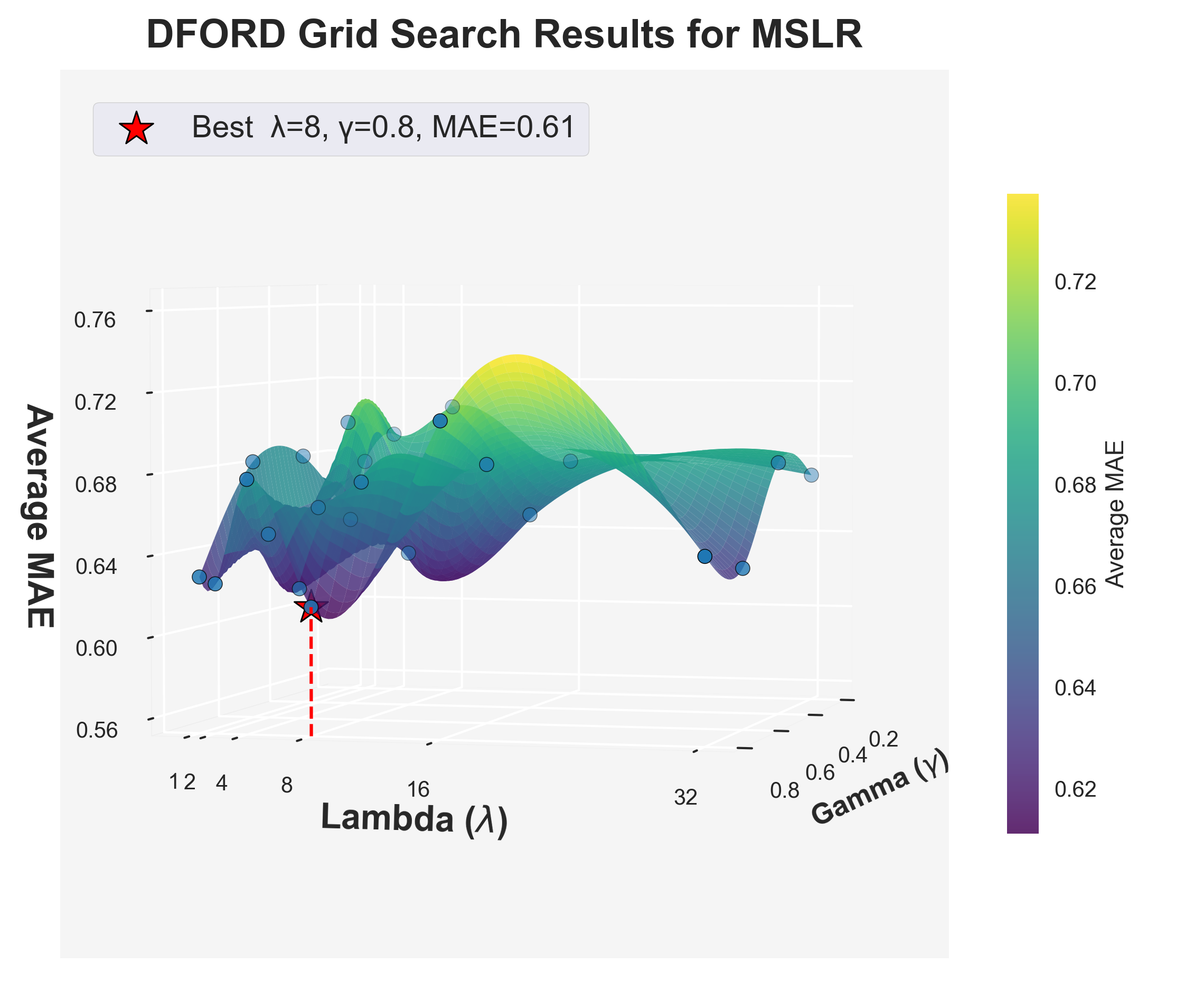}
    \includegraphics[width=0.44\columnwidth]{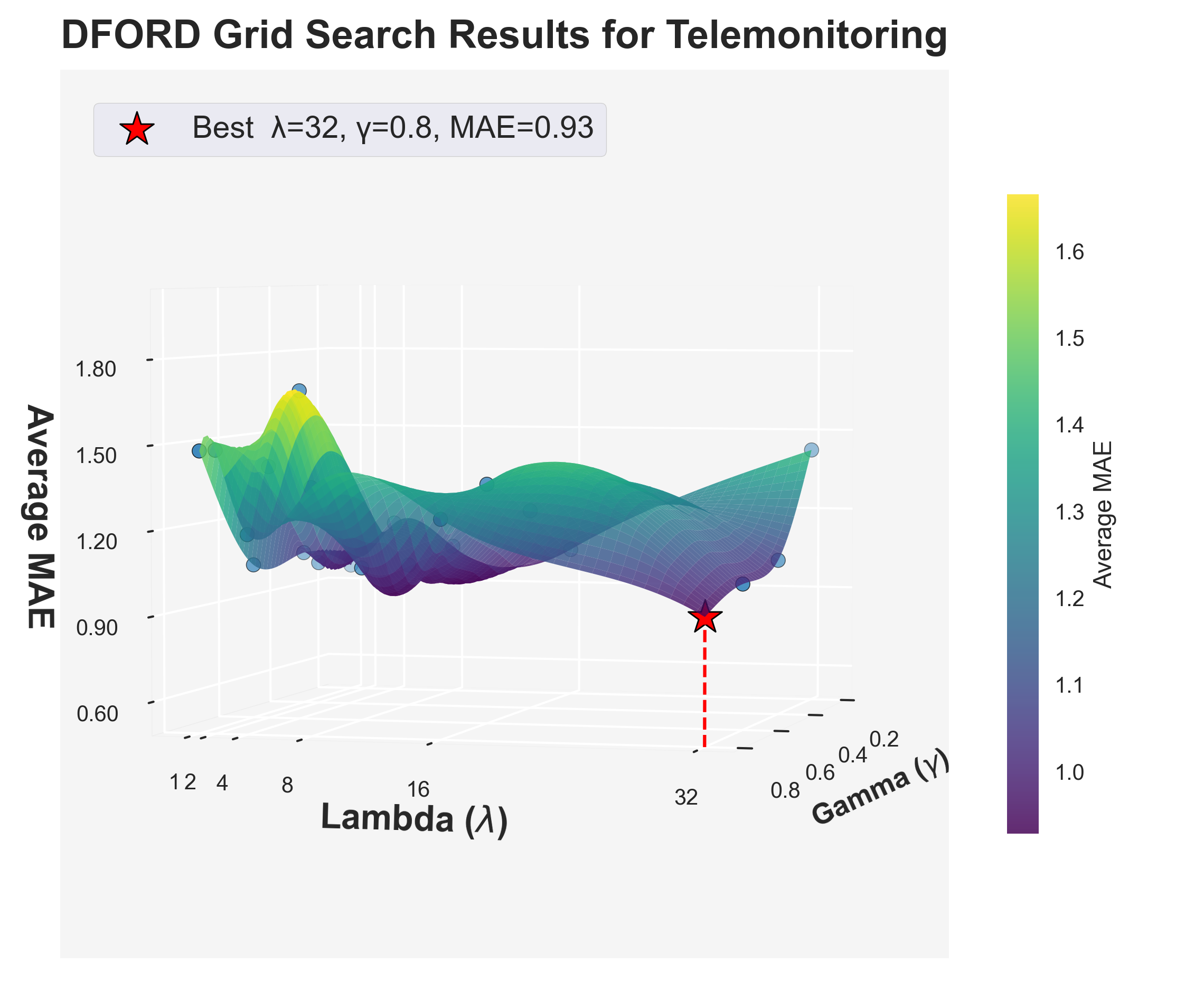} 
    \\
    \includegraphics[width=0.44\columnwidth]{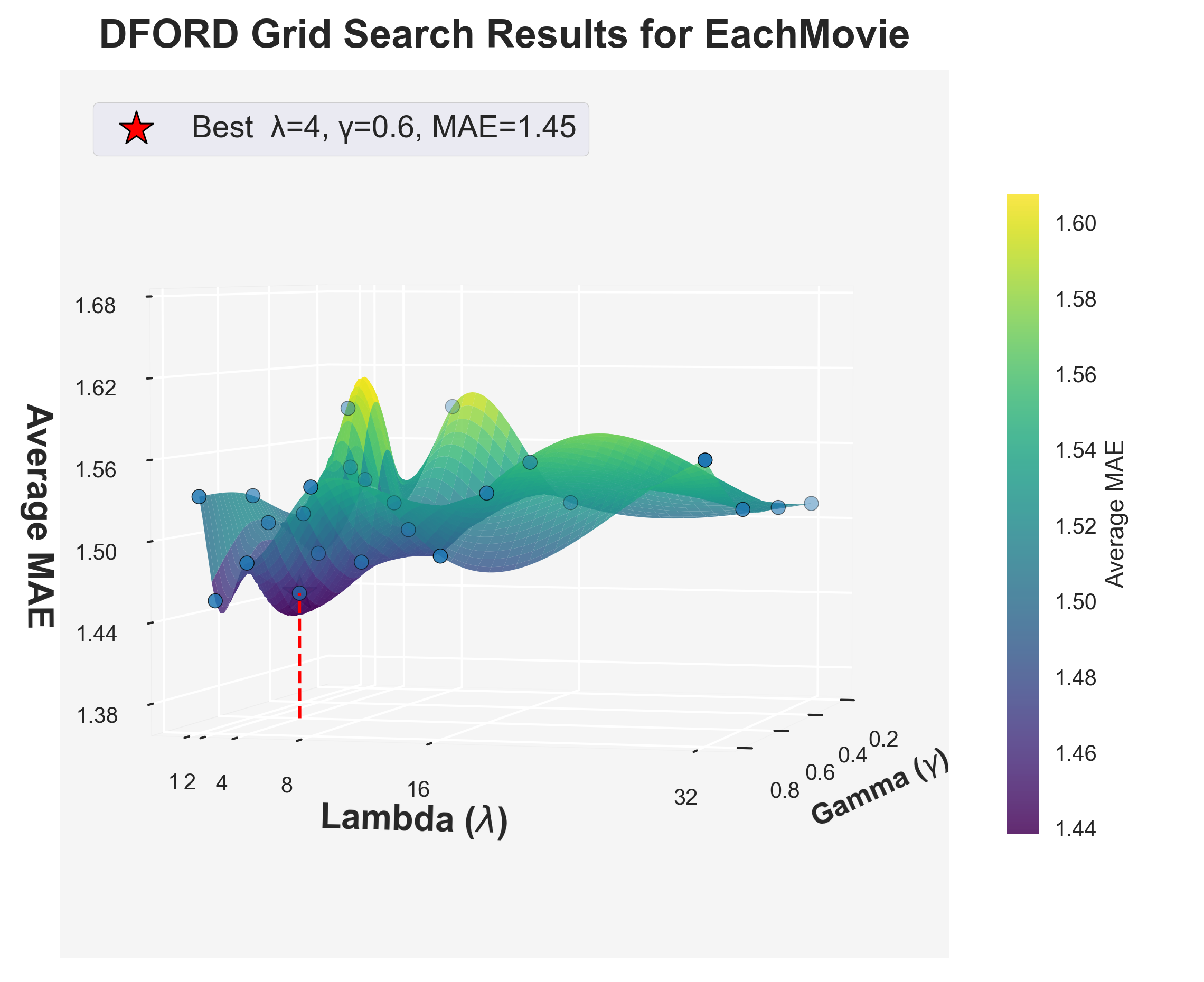}
    \includegraphics[width=0.44\columnwidth]{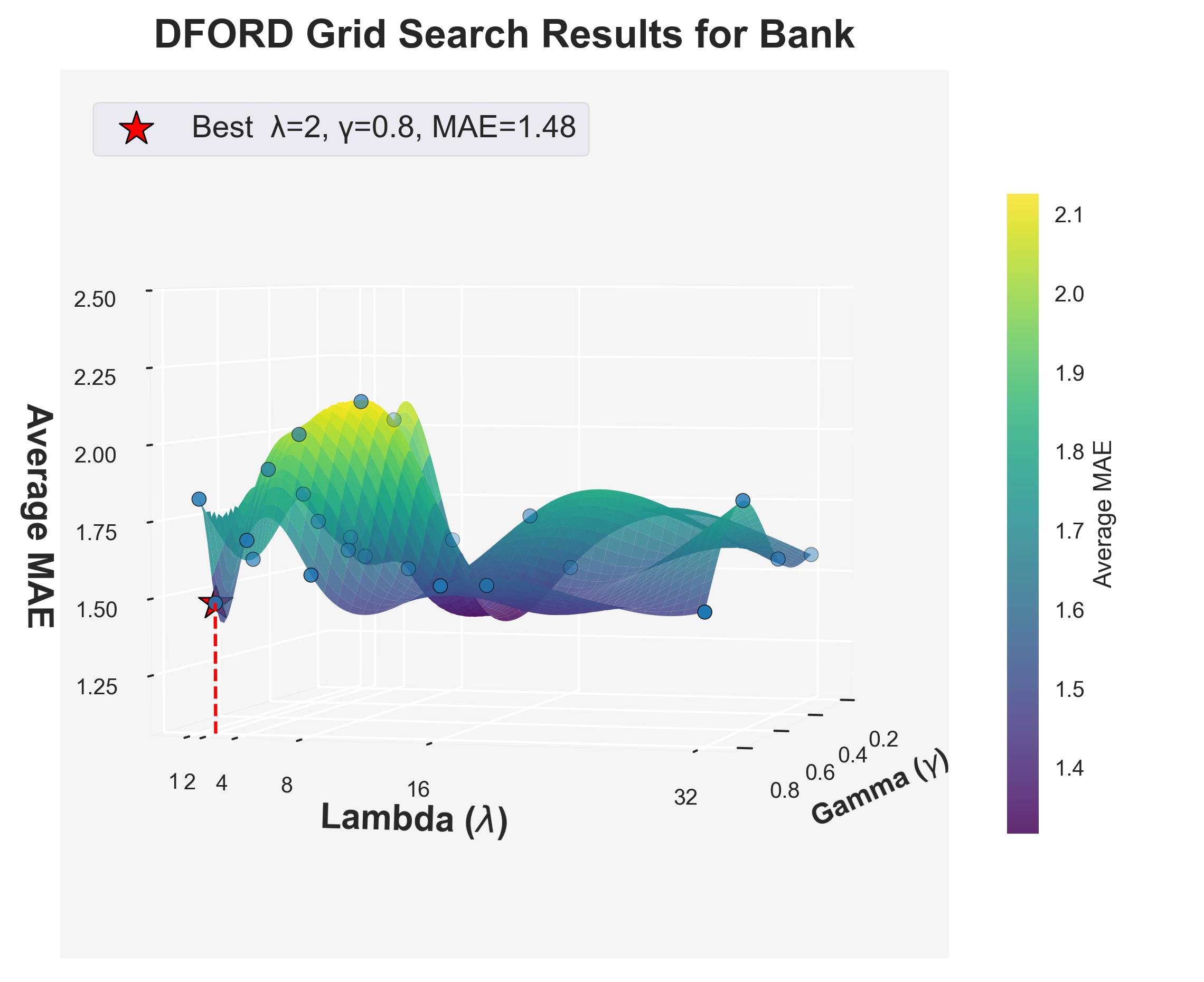}
    \\
    \includegraphics[width=0.44\columnwidth]{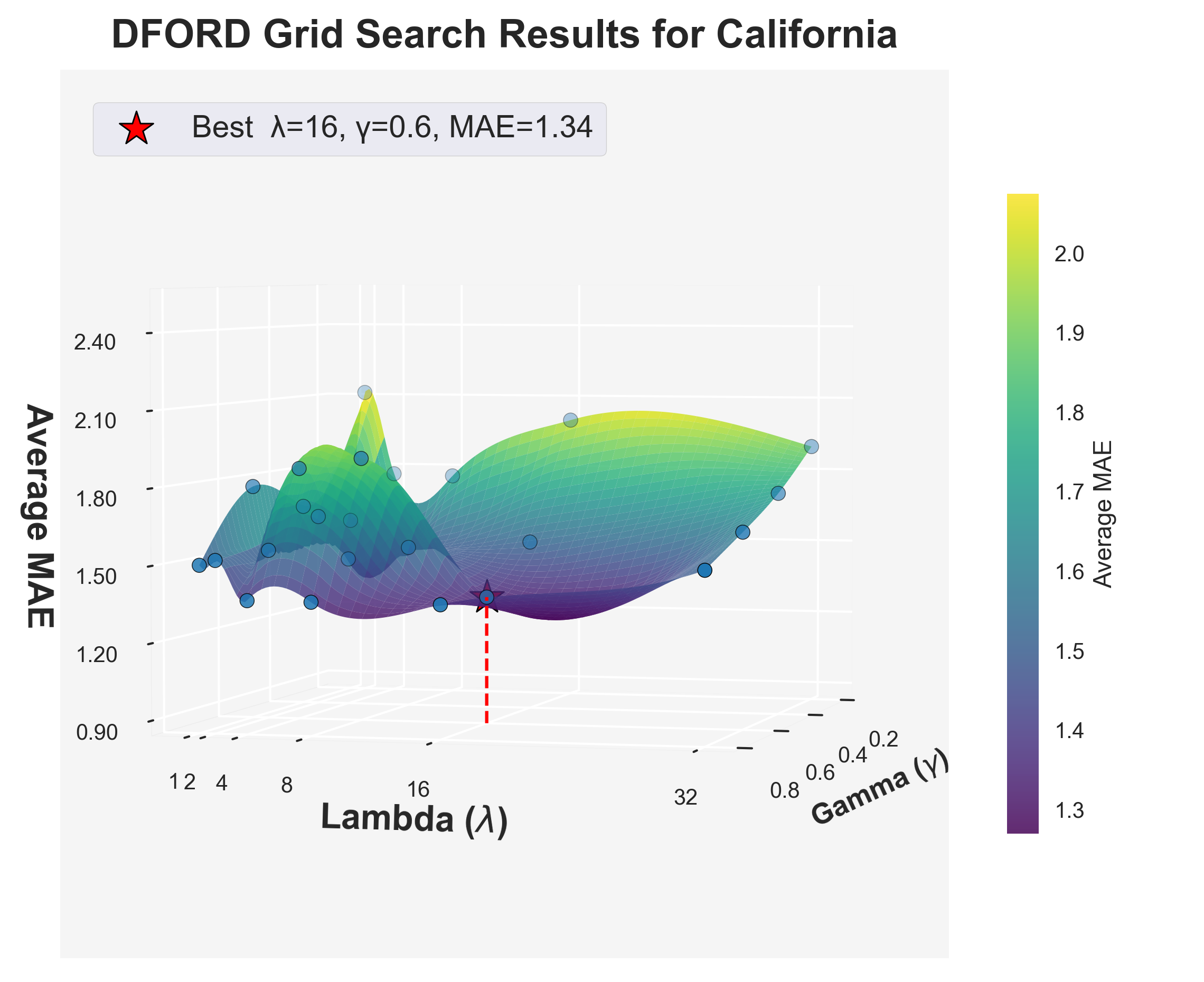}
    \includegraphics[width=0.44\columnwidth]{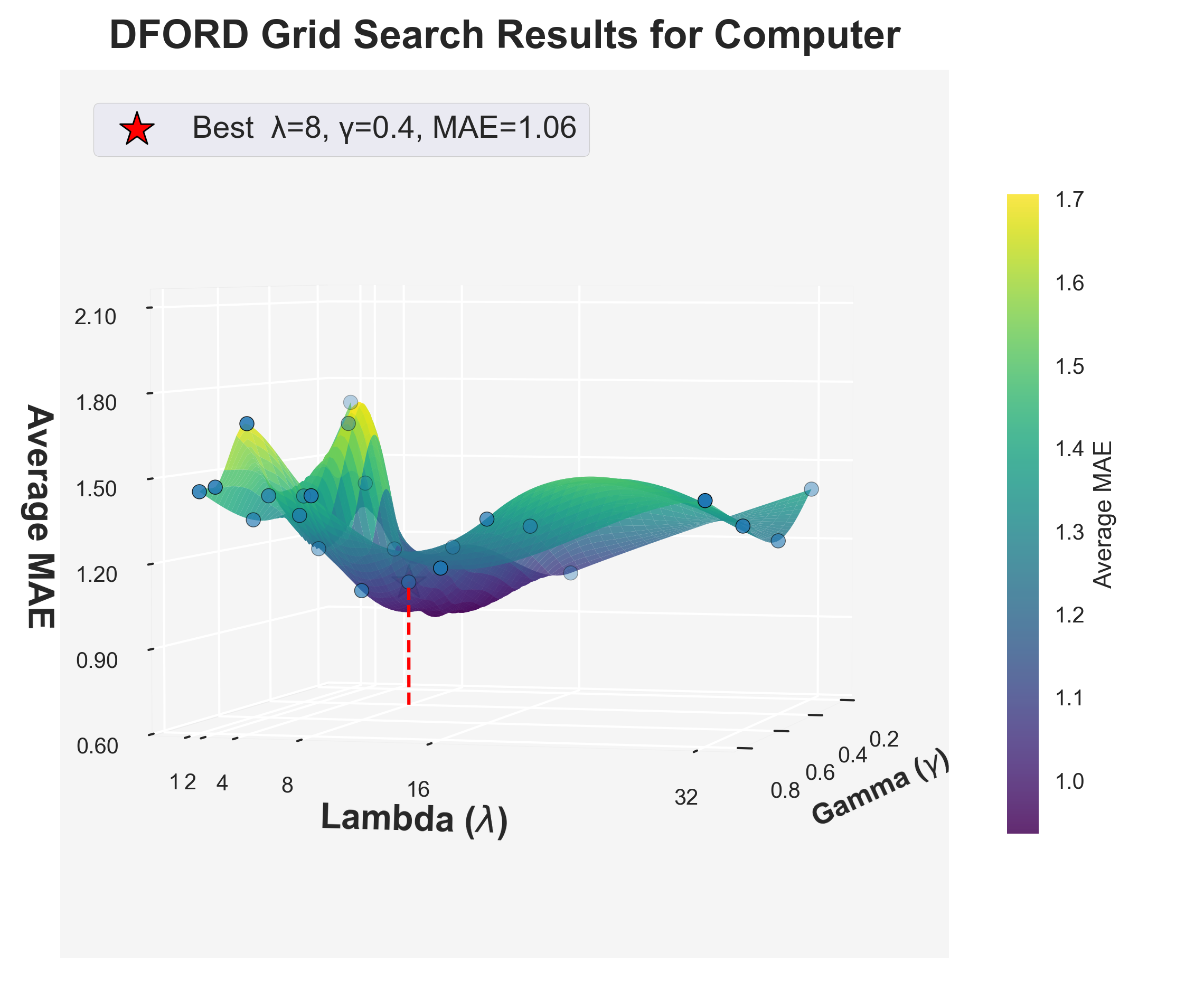}
    \caption{Hyperparameter Search ($\lambda$ and $\gamma$) for all datasets}
    \label{fig:hyperparameter-search}
\end{figure}
Figure \ref{fig:hyperparameter-search} below illustrates the grid search performed over a range of values of $\lambda$ $(1, 2, 4, 8, 16, 32)$  and $\gamma$ $(0.2, 0.4, 0.6, 0.8)$. Corresponding to each point is the average converged MAE. The point marked by a star ($\star$) is chosen to be the optimal parameter. For each pair of $\lambda$ and $\gamma$, we ran the algorithm $5$ times, $300,000$ iterations in each run, to get an averaged value of the converged loss. We choose the pair of $\lambda$ and $\gamma$ values that result in the minimum converged average MAE loss. Table~\ref{tab:dford-hyperparameters} shows optimal choice of $\lambda$ and $\delta$ used by DFORD algorithm for different datasets.
\begin{table}[h]
\centering
\begin{tabular}{|l|l|l|l|l|}
\hline
\textbf{Dataset}   & \textbf{Kernel}                                                 & \textbf{Clipping ($\alpha$)} & \textbf{$\lambda$} & \textbf{$\gamma$} \\ \hline
\textbf{EachMovie}          & Linear & 4.0  & 4  & 0.6 \\ \hline
\textbf{Telemonitoring}     & Linear & 13.0 & 32 & 0.8 \\ \hline
\textbf{MSLR Web 10K}       & Linear & 7.0  & 16 & 0.4 \\ \hline
\textbf{California Housing} & Linear & 10.0 & 16 & 0.8 \\ \hline
\textbf{Computer Activity}  & Linear & 10.0 & 8  & 0.4 \\ \hline
\textbf{Bank}               & Linear & 7.0  & 32 & 0.8 \\ \hline
\textbf{Abalone}   & $\left(1 + \langle \mathbf{x}_1, \mathbf{x}_2 \rangle\right)^3$ & 11.0                         & 32                     & 0.2          \\ \hline
\textbf{Synthetic} & $\left(1 + \langle \mathbf{x}_1, \mathbf{x}_2 \rangle\right)^2$ & 11.0                         & 4                      & 0.2          \\ \hline
\end{tabular}%
\caption{Hyperparameters used by DFORD}
\label{tab:dford-hyperparameters}
\end{table}


\subsection{Choosing Truncation Parameter $\delta$: }Figure \ref{fig:Trunc} shows the error curves (MAE vs Epoch) for different truncation values ($\delta = 100, 500, 750, 1000, 2000$) used for the Abalone ($\lambda = 32, \gamma = 0.2$) and Synthetic dataset ($\lambda = 4, \gamma = 0.2$). We can see that by increasing $\delta$, the average MAE decreases and reaches the minimum value, and by increasing $\delta$ further, the average MAE starts increasing again. We use the $\delta$ value, which gives the minimum average MAE.

\begin{figure}[H]
\centering
    \includegraphics[width=0.49\columnwidth]{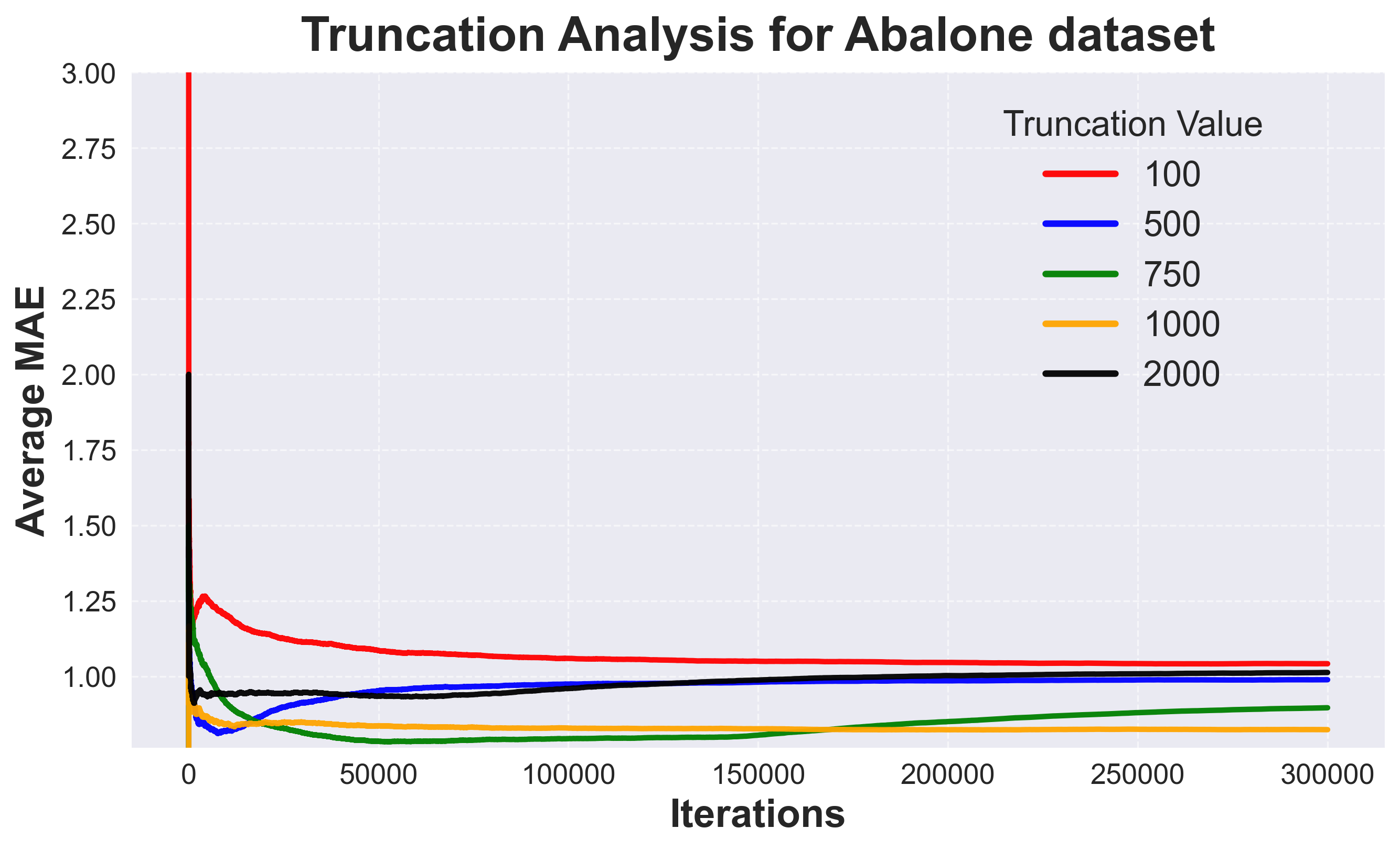}
    \includegraphics[width=0.49\columnwidth]{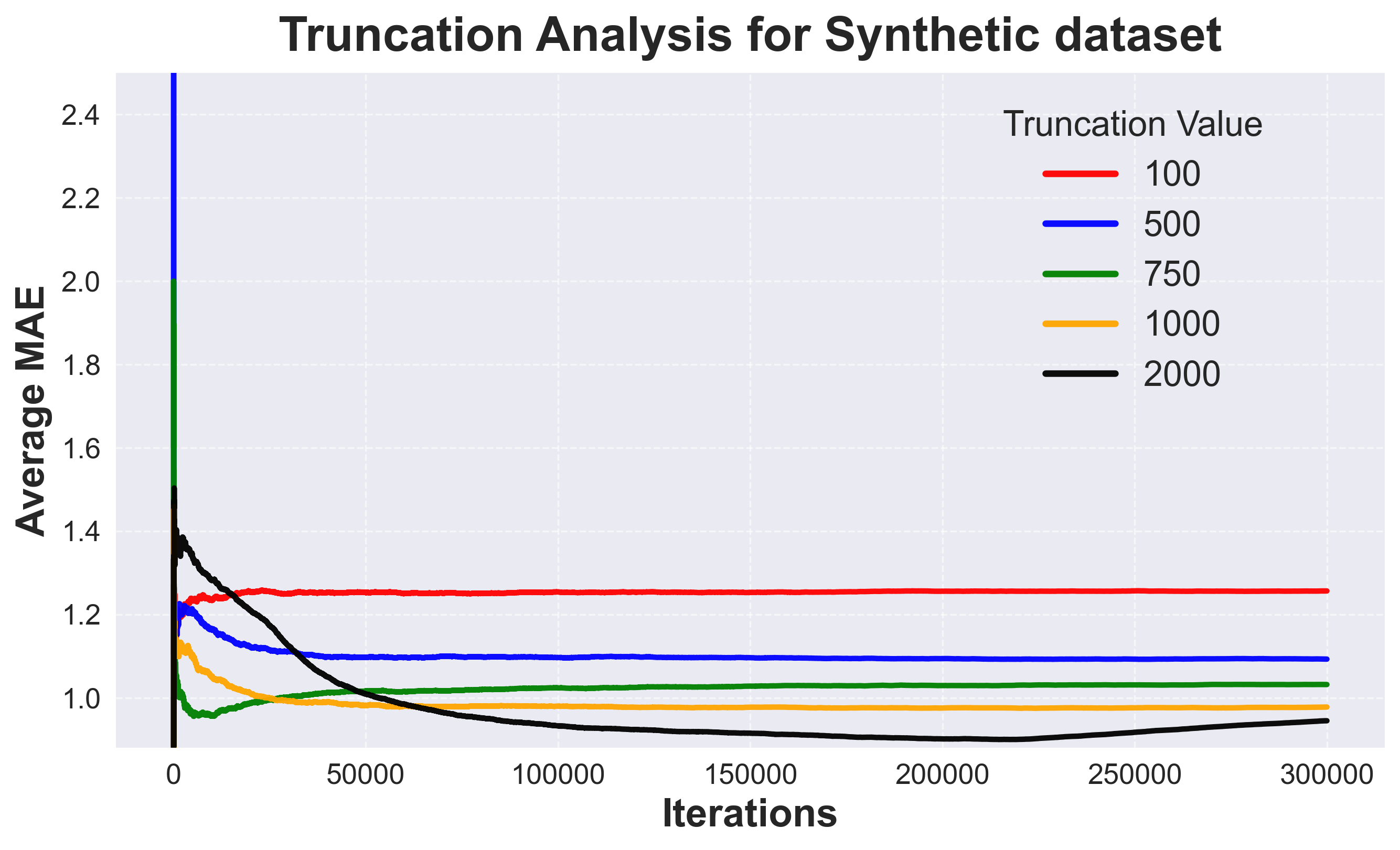}
    \caption{Searching for truncation parameter $(\delta)$ for Abalone and Synthetic dataset}
\label{fig:Trunc}
\end{figure}
\subsection{Experimental Results: }After choosing optimal hyperparameters, we ran the algorithm for each dataset for $10$ runs with $3\times10^6$ iterations in each run. All codes are written in Python, mainly using the NumPy library. All the experiments were done on a MacBook Air with an M1 chip and 8GB RAM.  

Figure \ref{fig:error} shows the error curves for all the datasets. We observe that DFORD outperforms PRIL (which is a partial label based approach) for almost all but 2 datasets. We see that for Abalone, MSLR Web 10K, EachMovie, Bank and Computer Activity datasets, DFORD achieves better or equal lowest converged average MAE compared to PRank as well (a full information approach). For MSLR Web 10K and EachMovie, the average MAE of DFORD is lower by a considerable margin. For Telemonitoring, Bank and California Housing, the Average MAE of DFORD is more than PRank, but by a very small margin. Overall, we can see that DFORD performs comparably to PRank. These observations support the argument that directional feedback can lead to efficient ordinal regression models.

The reason behind DFORD performing better than PRank for some datasets is as follows. In the cases when the intervals corresponding to two consecutive classes are overlapping, PRank puts its best bet on $\hat{y}^t$, which might be wrong, and it does not have the option to choose another label. On the other hand, DFORD allows us to explore all the labels. This exploration might pick the correct label sometimes.


\begin{figure}[t]
\centering
  \includegraphics[width=0.49\columnwidth]{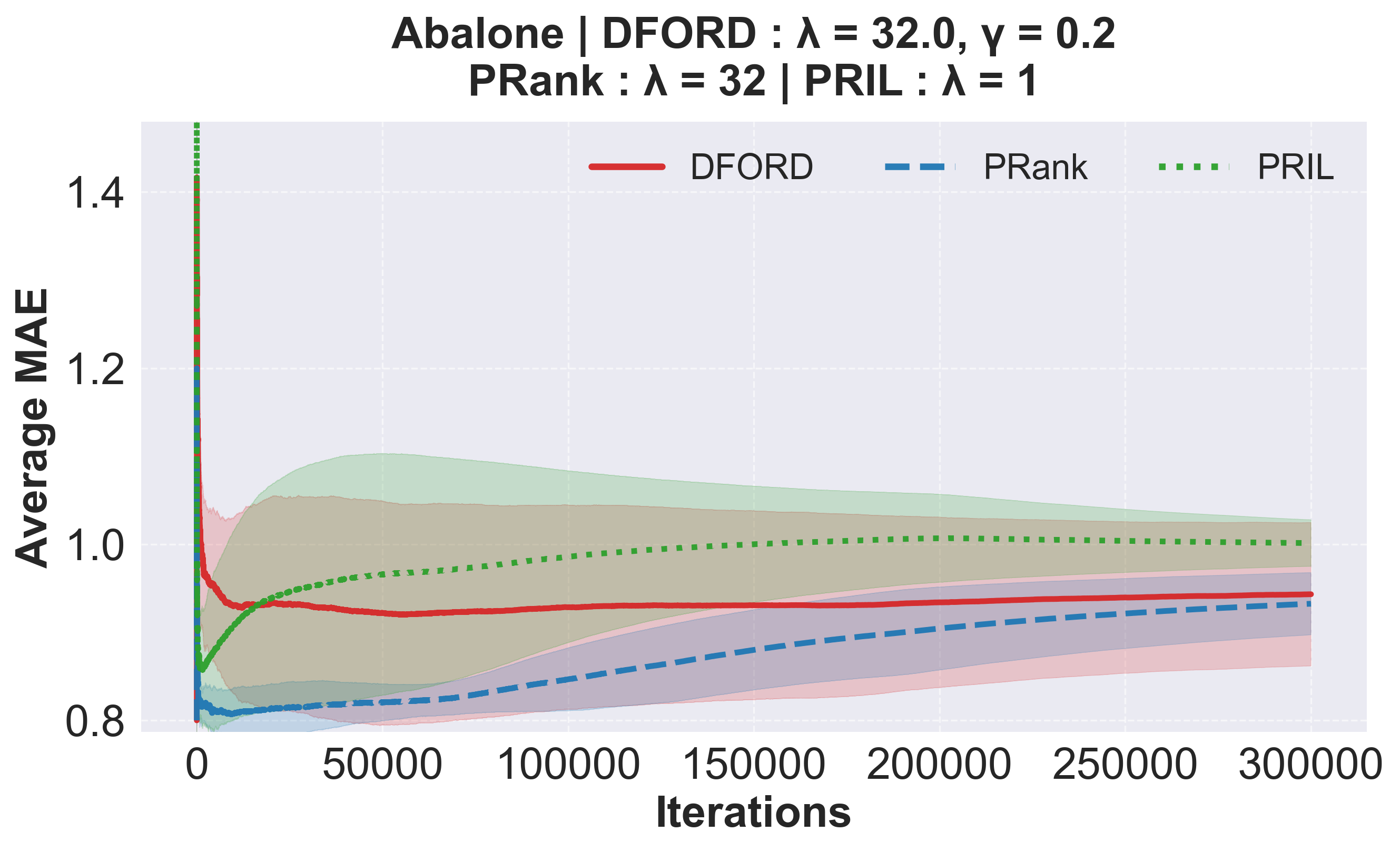}
  \includegraphics[width=0.49\columnwidth]{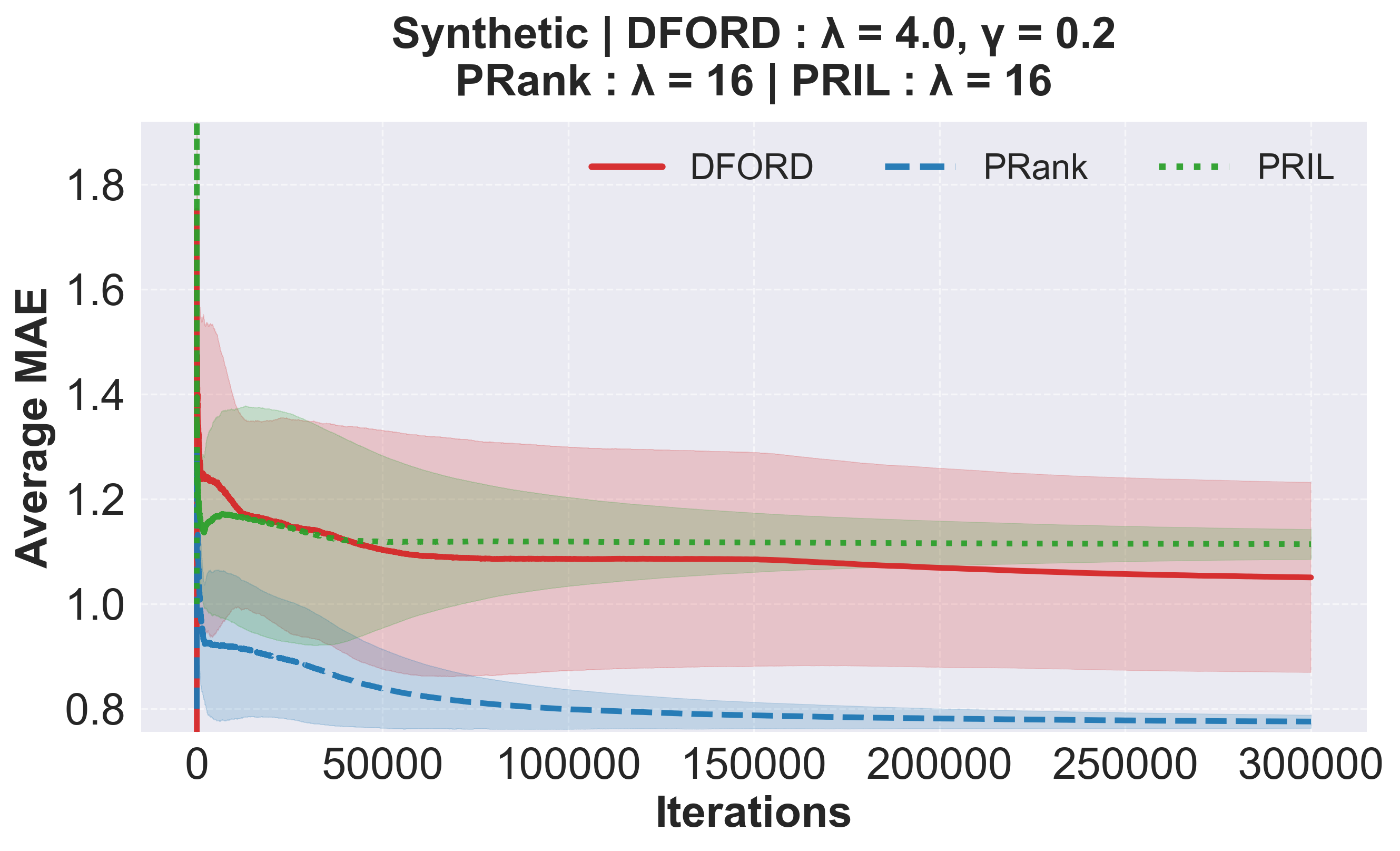}
  \\
  \includegraphics[width=0.49\columnwidth]{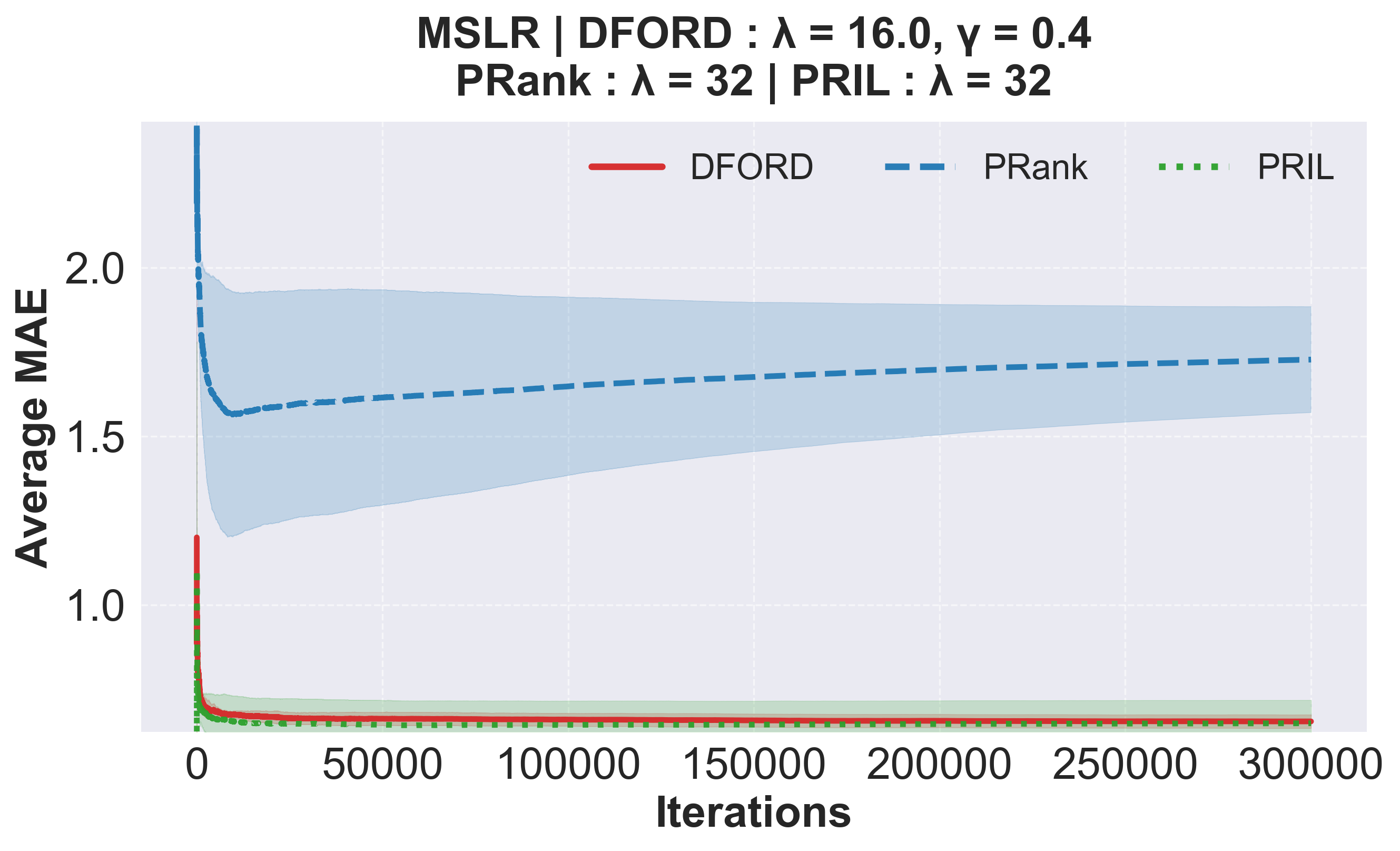}
  \includegraphics[width=0.49\columnwidth]{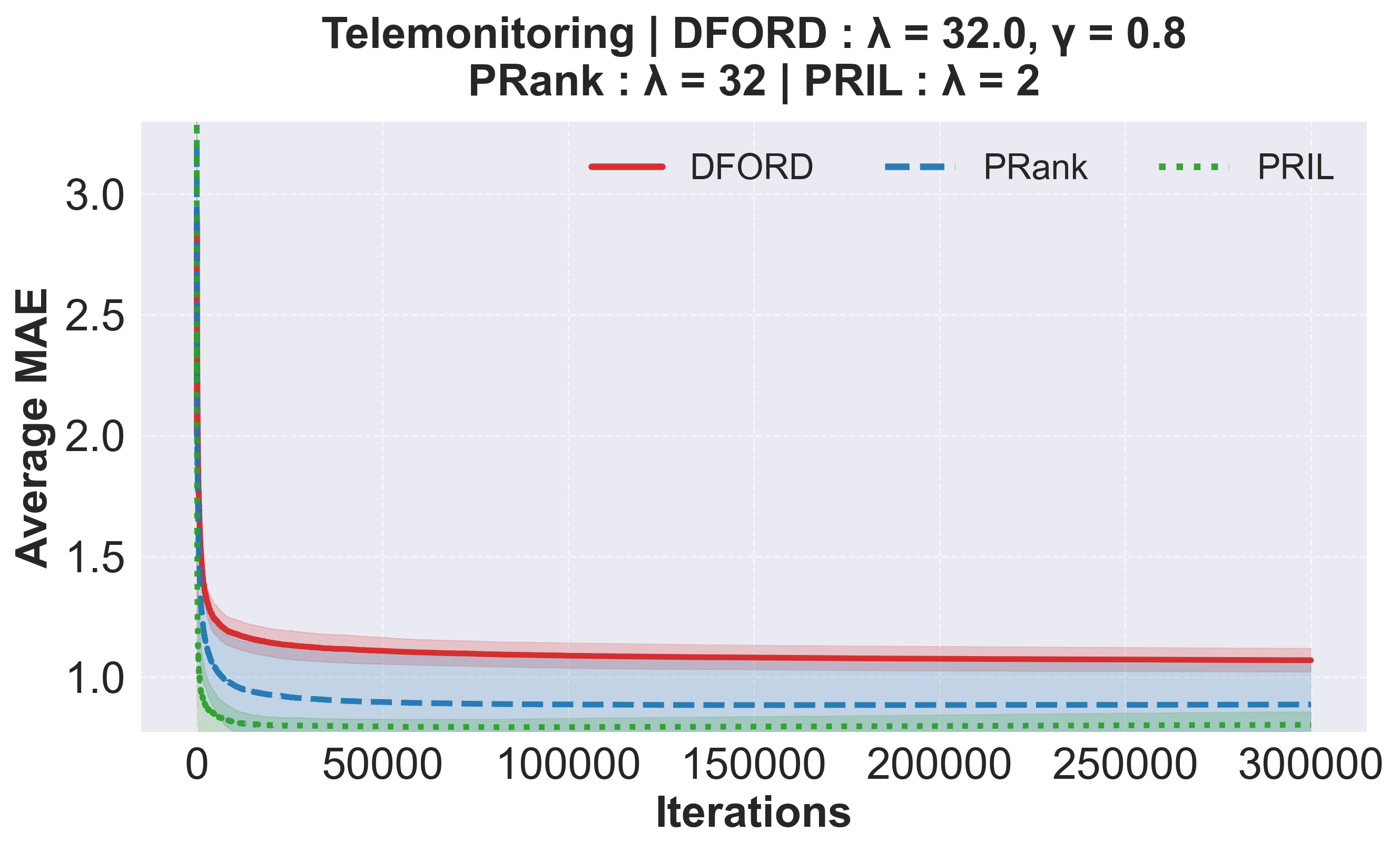}
  \\
  \includegraphics[width=0.49\columnwidth]{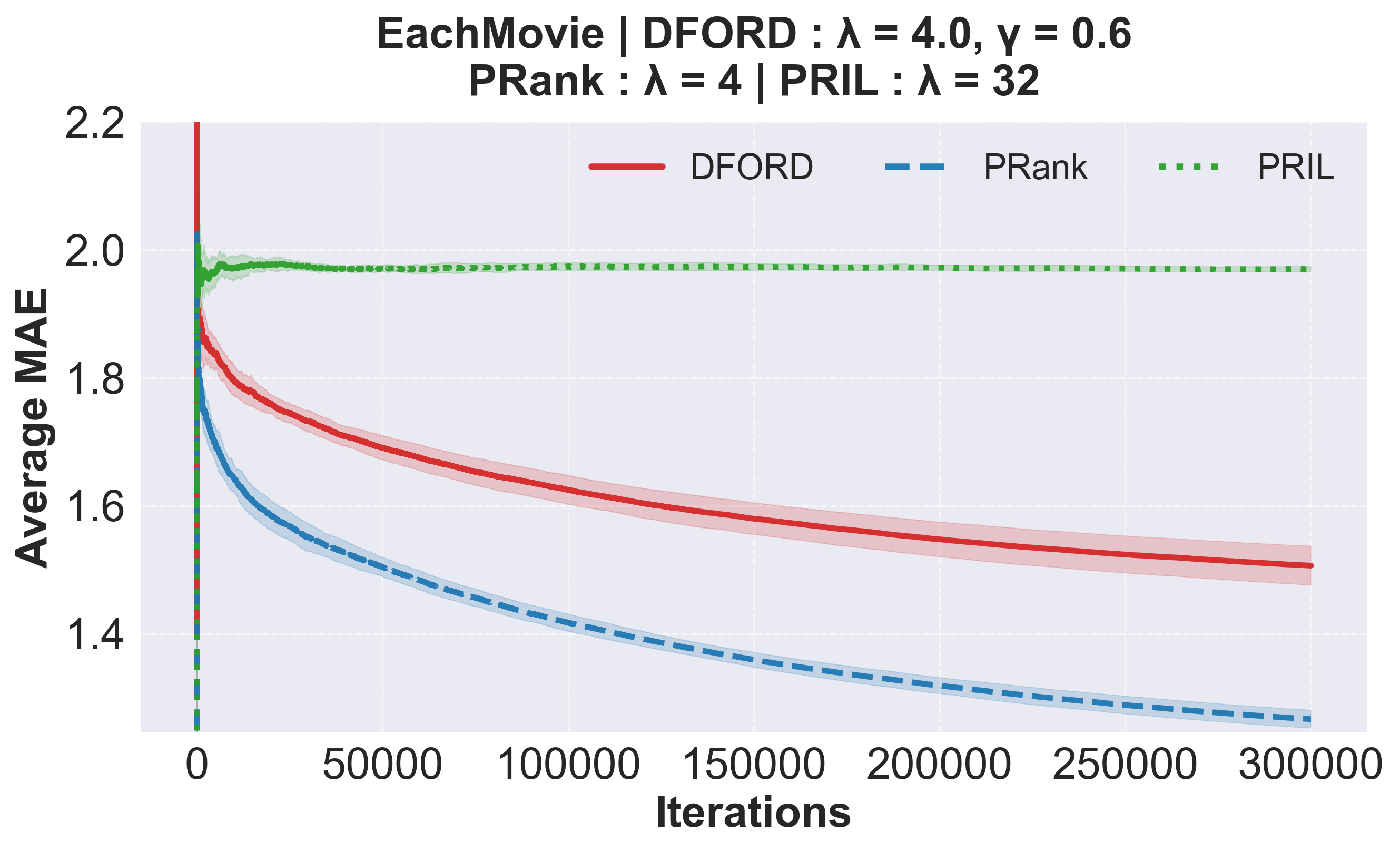}
  \includegraphics[width=0.49\columnwidth]{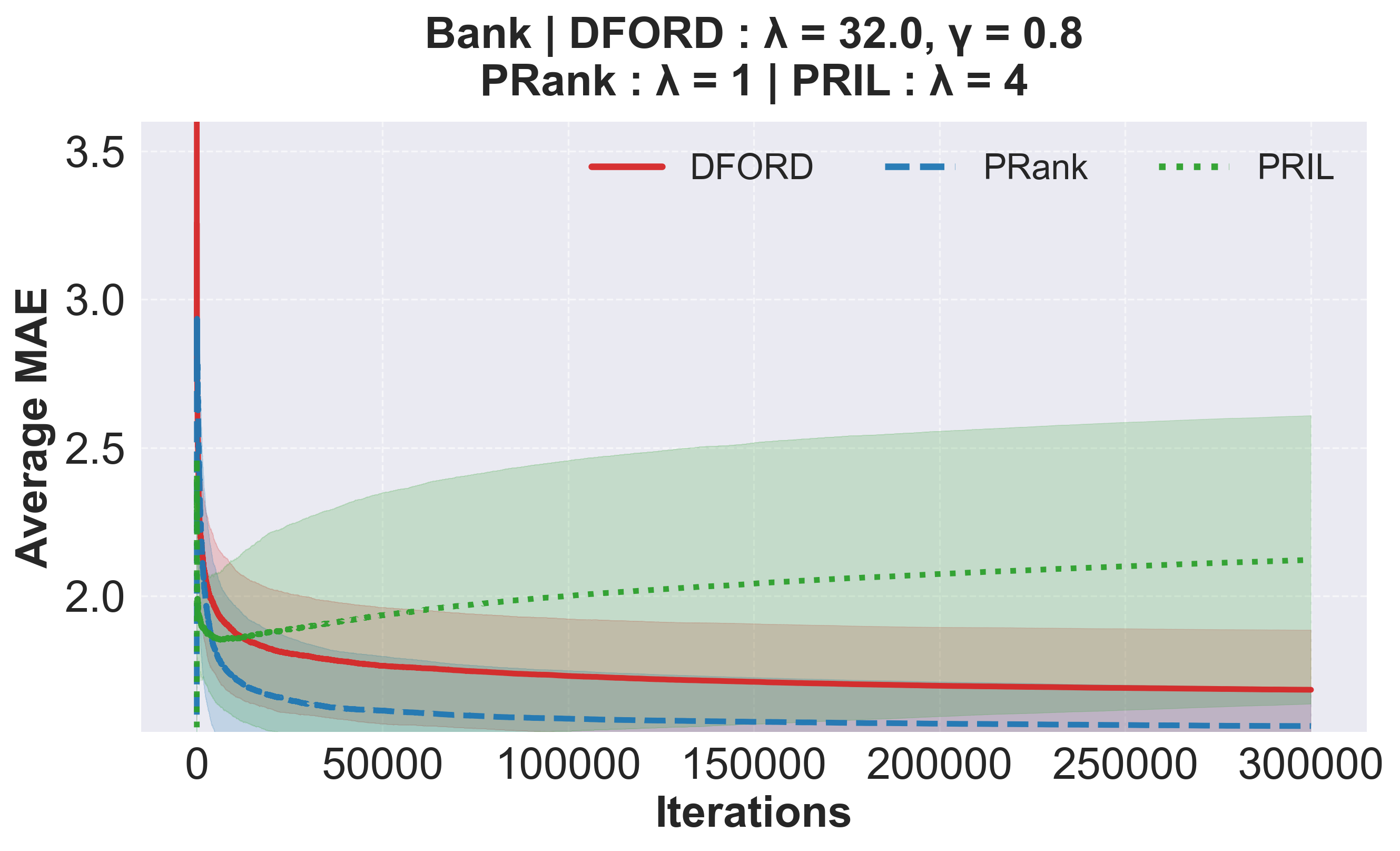}
  \\
  \includegraphics[width=0.49\columnwidth]{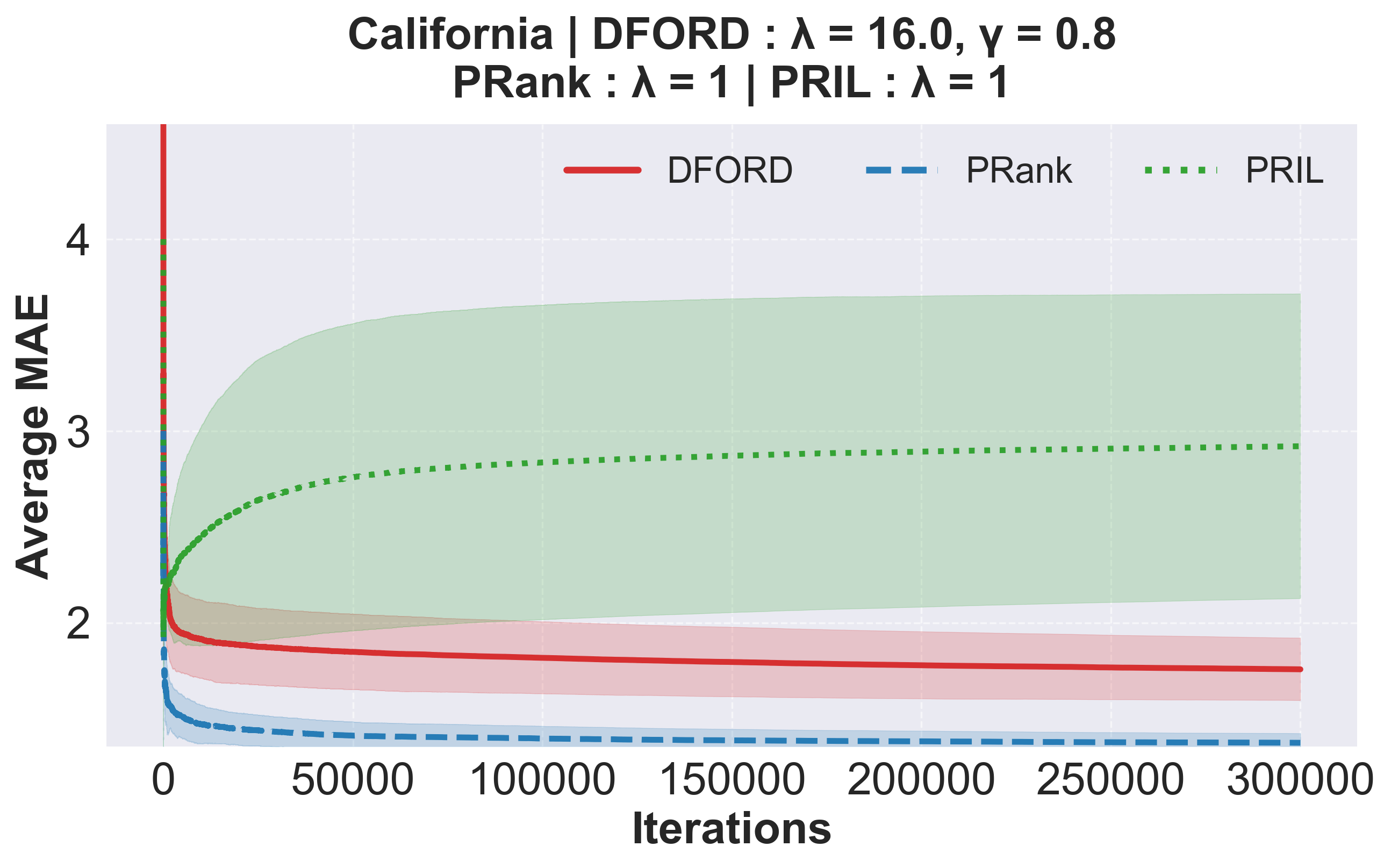}
  \includegraphics[width=0.49\columnwidth]{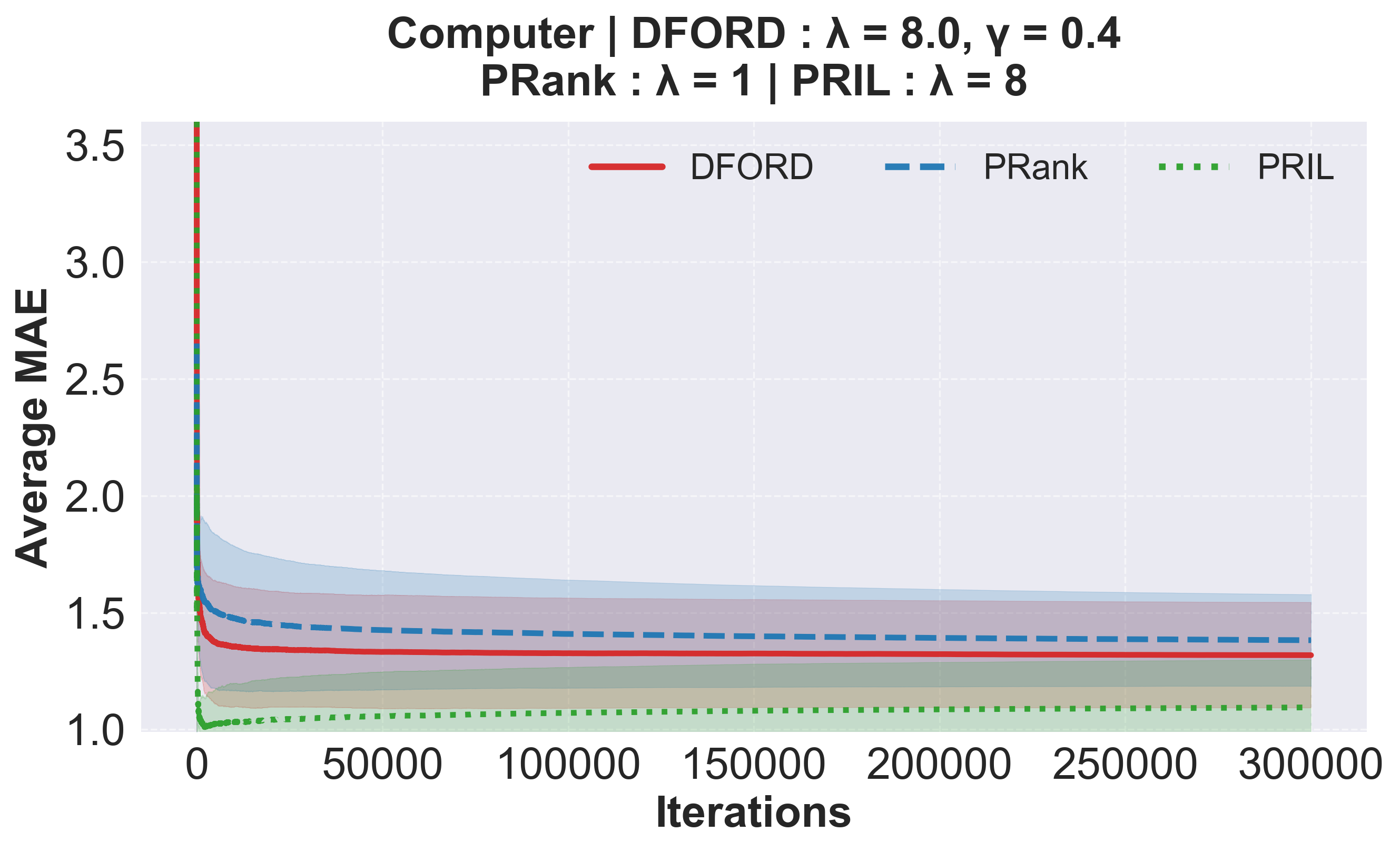}
\caption{Average Mean Absolute Error (MAE) curves for various datasets.}
\label{fig:error}
\end{figure}

\subsection*{On the Effect of Label Sampling on DFORD in Initial Rounds}
We conducted a small experiment to understand how much error can be caused in the initial rounds if we always choose the predicted label $\hat{y}^t$. We run the algorithm by selecting $\gamma=0$ (exploitation mode) for T=10000. We report average MAE results recorded at $t=1000,2000,\ldots, 10000$. Then, we do the same for choosing $\gamma=0.4$ and $\gamma=0.8$ (exploration-exploitation mode). The results are shown in Table \ref{tab:label-sampling}. Using a nonzero $\gamma$ mitigates the high error issues if we always choose $\hat{y}^t$. Thus, the use of equation (\ref{eq:sampling-prob}) for sampling classes leads to better learning even in initial rounds.

\begin{table}[]
\centering
\begin{tabular}{|l|l|l|}
\hline
\textbf{Dataset} & \textbf{$\gamma$} & \textbf{Average MAE}                                                       \\ \hline
\multicolumn{1}{|c|}{}  & 0   & {[}2.014, 1.979, 1.995, 1.992, 1.976, 1.976, 1.977, 1.987, 1.982, 1.994{]} \\ \cline{2-3} 
\textbf{EachMovie}      & 0.4 & {[}1.877, 1.873, 1.852, 1.864, 1.851, 1.842, 1.84, 1.843, 1.839, 1.839{]}  \\ \cline{2-3} 
                 & 0.8         & {[}1.805, 1.843, 1.819, 1.82, 1.819, 1.805, 1.805, 1.802, 1.795, 1.796{]}  \\ \hline
                 & 0           & {[}3.247, 3.437, 3.493, 3.506, 3.525, 3.537, 3.543, 3.552, 3.549, 3.561{]} \\ \cline{2-3} 
\textbf{California}     & 0.4 & {[}1.729, 1.642, 1.624, 1.614, 1.605, 1.604, 1.596, 1.583, 1.585, 1.577{]} \\ \cline{2-3} 
                 & 0.8         & {[}1.713, 1.634, 1.627, 1.616, 1.611, 1.612, 1.604, 1.593, 1.597, 1.59{]}  \\ \hline
                 & 0           & {[}2.151, 2.45, 2.576, 2.649, 2.685, 2.722, 2.747, 2.78, 2.801, 2.808{]}   \\ \cline{2-3} 
\textbf{Telemonitoring} & 0.4 & {[}1.101, 0.959, 0.911, 0.884, 0.881, 0.885, 0.885, 0.881, 0.878, 0.882{]} \\ \cline{2-3} 
                 & 0.8         & {[}1.167, 1.047, 0.992, 0.955, 0.948, 0.95, 0.945, 0.934, 0.928, 0.931{]}  \\ \hline
\end{tabular}%
\caption{Effect of label sampling}
\label{tab:label-sampling}
\end{table}

\section{Conclusion}
In this paper, we motivate the importance of directional feedback, a type of weak supervision setting relevant to ordinal regression tasks. We proposed an algorithm for online ordinal regression with directional feedback. We also provide theoretical guarantees on the expected regret bound of the online algorithm. Our algorithm achieves an expected regret bound of $\mathcal{O}(\log T)$. We also perform numerical experiments on real-world and synthetic datasets. Our algorithm performed comparably to the full information baseline algorithm and even outperformed it for some datasets. Our findings suggest that learning directional feedback can lead to equally efficient ordinal regression models compared to full information-based models.

\subsubsection*{Acknowledgments}
Naresh Manwani gratefully acknowledges support from ANRF under grant CRG/2023/007970.  

\bibliography{main}
\bibliographystyle{tmlr}

\appendix
\section{MISSING PROOFS}
\subsection{Proof of Theorem 1}
Let $d_t=\mathbb{I}[\tilde{y}^t<y^t]$. Taking expectation of $\taut_i^t$ with respect to $P^t$, we get
    \begin{align*}
    &\mathbb{E}_{P^t}[\tilde{\tau}_i^t]=\mathbb{E}_{P^t}\left[\frac{1}{P^t(i)}d_t\mathbb{I}[i=\tilde{y}^t]\mathbb{I}[\tilde{z}_i^t(\ww^t.\xx^t -\theta_i^t)\leq 0]\right]-\mathbb{E}\left[\frac{1}{P^t(i)}(1-d_t)\mathbb{I}[i=\tilde{y}^t]\mathbb{I}[\tilde{z}_i^t(\ww^t.\xx^t -\theta_i^t)\leq 0]\right]\\
    &=\sum_{j=1}^{K}\frac{P^t(j)}{P^t(i)}\mathbb{I}[j<y^t]\mathbb{I}[i= j]\mathbb{I}\left[\frac{\mathbb{I}[j<y^t]\mathbb{I}[i= j]-\mathbb{I}[j\geq y^t]\mathbb{I}[i= j]}{P^t(i)}(\ww^t.\xx^t -\theta_i^t)\leq 0\right]\\
    &\;\;\;-\sum_{j=1}^{K}\frac{P^t(j)}{P^t(i)}\mathbb{I}[j\geq y^t]\mathbb{I}[i= j]\mathbb{I}\left[\frac{\mathbb{I}[j<y^t]\mathbb{I}[i= j]-\mathbb{I}[j\geq y^t]\mathbb{I}[i= j]}{P^t(i)}(\ww^t.\xx^t -\theta_i^t)\leq 0\right]\\
    &=\mathbb{I}[i<y^t]\mathbb{I}\left[\frac{\mathbb{I}[i<y^t]-\mathbb{I}[i\geq y^t]}{P^t(i)}(\ww^t.\xx^t -\theta_i^t)\leq 0\right]-\mathbb{I}[i\geq y^t]\mathbb{I}\left[\frac{\mathbb{I}[i<y^t]-\mathbb{I}[i\geq y^t]}{P^t(i)}(\ww^t.\xx^t -\theta_i^t)\leq 0\right]\\
    &=\left[\mathbb{I}[i<y^t]-\mathbb{I}[i\geq y^t]\right]\mathbb{I}\left[\frac{z_i^t}{P^t(i)}(\ww^t.\xx^t -\theta_i^t)\leq 0\right]\\
    &=z_i^t\mathbb{I}\left[z_i^t(\ww^t.\xx^t -\theta_i^t)\leq 0\right]=\tau_i^t.
    \end{align*}

\subsection{Proof for Lemma 1}
We see that,
\begin{align*}
\theta_{i+1}^{t+1} &- \theta_i^{t+1} 
= (1-\eta_t \lambda)(\theta_{i+1}^t-\theta_i^t) - \eta_t(\tilde{\tau}_{i+1}^t -\tilde{\tau}_{i}^t)
\end{align*}
where $\eta_t=\frac{1}{\lambda t}$. We consider the following different cases to show that DFORD-Linear preserves the ordering of the thresholds in the expected sense. Let us use the notation $P^{[t]}$ for $P^1P^2 \ldots P^t$ which is joint distribution of $\tilde{y}^1,\ldots,\tilde{y}^t$. \begin{enumerate}
\item {\bf  Case 1: $\mathbb{I}[\tilde{y}^t<y^t]=1$}. Here, we analyse different cased separately.
\begin{enumerate}
\item $i\in\{1,\ldots,\tilde{y}^t-2\}$: In this case
$\tilde{\tau}_{i+1}^t =\tilde{\tau}_{i}^t=0$. Thus, \begin{align*}
\theta_{i+1}^{t+1} &- \theta_i^{t+1} 
= (\theta_{i+1}^t-\theta_i^t).
\end{align*}
Taking expectation with respect to $P^{[t]}$ on both sides, we get
\begin{align*}
\Ex_{P^{[t]}}[\theta_{i+1}^{t+1} - \theta_i^{t+1} ]
&= \Ex_{P^{[t]}}[(\theta_{i+1}^t-\theta_i^t)]=\Ex_{P^{[t-1]}}[(\theta_{i+1}^t-\theta_i^t)]
\end{align*}
Since $\mathbb{E}_{P^{[t-1]}}[\theta_{i+1}^t-\theta_i^t]\geq 0$, we have $\mathbb{E}_{P^{[t]}}[\theta_{i+1}^{t+1}-\theta_i^{t+1}]\geq  0$.
\item $i=\tilde{y}^t-1$: In this case, $\tilde{\tau}_{\tilde{y}^t-1}^t=0$. Thus, 
\begin{align}
\label{eq:diff1}
\theta_{\tilde{y}^t}^{t+1} &- \theta_{\tilde{y}^t-1}^{t+1} 
= ( \theta_{\tilde{y}^t}^{t} - \theta_{\tilde{y}^t-1}^{t} ) -\eta_t \tilde{\tau}_{\tilde{y}^t}^t
\end{align}
We see that $\tilde{z}_{\tilde{y}^t}^t=\frac{1}{P_t(\tilde{y}^t)}$ and $\tilde{\tau}_{\tilde{y}^t}^t=\frac{1}{P_t(\tilde{y}^t)}\mathbb{I}[\frac{1}{P_t(\tilde{y}^t)}(\ww^t.\xx^t-\theta_{\tilde{y}^t})\leq 0]$. Thus, $\mathbb{E}_{P^t}[\tilde{\tau}^t_{\tilde{y}^t}]=\sum_{k=1}^KP^t(k)\frac{1}{P^t(k)}\mathbb{I}[\frac{1}{P^t(k)}(\ww^t.\xx^t-\theta_k^t)\leq 0]=\sum_{k=1}^K \mathbb{I}[\frac{1}{P^t(k)}(\ww^t.\xx^t-\theta_k^t)\leq 0]\leq K$. Taking expectations on both sides with respect to $P^{[t]}$ of eq.(\ref{eq:diff1}), we have
\begin{align*}
     \mathbb{E}_{P^{[t]}}[  \theta_{\tilde{y}^t}^{t+1} - \theta_{\tilde{y}^t-1}^{t+1} ]
&= (1-\eta_t \lambda)\mathbb{E}_{P^{[t]}}[ \theta_{\tilde{y}^t}^{t} - \theta_{\tilde{y}^t-1}^{t} ]-\eta_t\mathbb{E}_{P^{[t]}}[\tilde{\tau}^t_{\tilde{y}^t}]\\
&=(1-\eta_t \lambda)\mathbb{E}_{P^{[t-1]}}[\theta_{\tilde{y}^t}^{t} - \theta_{\tilde{y}^t-1}^{t}]-\eta_t \mathbb{E}_{P^t}[\tilde{\tau}^t_{\tilde{y}^t}]\\
&\geq (1-\eta_t \lambda)\mathbb{E}_{P^{[t-1]}}[\theta_{\tilde{y}^t}^{t} - \theta_{\tilde{y}^t-1}^{t}]-K\eta_t
\end{align*} 
Since $\mathbb{E}_{P^{[t-1]}}[\theta_{\tilde{y}^t}^{t} - \theta_{\tilde{y}^t-1}^{t}]\geq\frac{K\eta_t}{1-\lambda\eta_t}$, we have $\mathbb{E}_{P^{[t]}}[\theta_{\tilde{y}^t}^{t+1} - \theta_{\tilde{y}^t-1}^{t+1}]\geq 0$.
\item $i=\tilde{y}^t$: In this case $\tilde{\tau}^{t}_{\tilde{y}^t+1}=0$ and $\mathbb{E}^{P^t}[\tilde{\tau}^t_{\tilde{y}^t}]=\sum_{k=1}^K\mathbb{I}[\frac{1}{P^t(k)}(\ww^t.\xx^t-\theta_k^t)\leq 0]$. Thus,
\begin{eqnarray*}
\theta_{\tilde{y}^t+1}^{t+1} - \theta_{\tilde{y}^t}^{t+1}  &=  (1-\eta_t\lambda)(\theta_{\tilde{y}^t+1}^t - \theta_{\tilde{y}^t}^t)+\eta_t\tilde{\tau}^t_{\tilde{y}^t} 
\end{eqnarray*}
Taking expectation on both sides, we get the following.
\begin{align*}
&\mathbb{E}_{P^{[t]}}[ \theta_{\tilde{y}^t+1}^{t+1} - \theta_{\tilde{y}^t}^{t+1} ] =  (1-\eta_t\lambda)\mathbb{E}_{P^{[t-1]}}[(\theta_{\tilde{y}^t+1}^t - \theta_{\tilde{y}^t}^t)]+\eta_t\mathbb{E}_{P^t}[\tilde{\tau}^t_{\tilde{y}^t} ]\\
&=(1-\eta_t\lambda)\mathbb{E}_{P^{[t-1]}}[(\theta_{\tilde{y}^t+1}^t - \theta_{\tilde{y}^t}^t)]+\eta_t\sum_{k=1}^K\mathbb{I}[\frac{1}{P^t(k)}(\ww^t.\xx^t-\theta_k^t)\leq 0]\\
&\geq K\eta_t+\eta_t\sum_{k=1}^K\mathbb{I}[\frac{1}{P^t(k)}(\ww^t.\xx^t-\theta_k^t)\leq 0]\geq 0
\end{align*}
\item $i\in \{\tilde{y}^t+1,\ldots,K-1\}$: In this case, $\tilde{\tau}^t_{i+1}=\tilde{\tau}^t_i=0$. Thus,
\begin{align*}
    \theta^{t+1}_{i+1}-\theta_i^{t+1}=(1-\eta_t \lambda)(\theta_{i+1}^t-\theta_i^t)
\end{align*}
and taking expectation on both sides, we get
\begin{align*}
    \mathbf{E}_{P^{[t]}}[\theta^{t+1}_{i+1}-\theta_i^{t+1}]=(1-\eta_t \lambda)\mathbf{E}_{P^{[t-1]}}[(\theta_{i+1}^t-\theta_i^t)]\geq K\eta_t>0
\end{align*}
\end{enumerate}
\item {\bf Case 2: $\mathbb{I}[\tilde{y}^t<y^t]=0$}. Which means $\tilde{y}^t\geq y^t$.
\begin{enumerate}
\item $i\in \{1,\ldots,\tilde{y}^t-2\}$: In this case
$\tilde{\tau}_{i+1}^t =\tilde{\tau}_{i}^t=0$. Thus, \begin{align*}
\theta_{i+1}^{t+1} &- \theta_i^{t+1} 
= (1-\eta_t \lambda)(\theta_{i+1}^t-\theta_i^t).
\end{align*}
Since $\mathbb{E}_{P^{[t-1]}}[\theta_{i+1}^t-\theta_i^t]\geq \frac{K\eta_t}{1-\lambda \eta_t}$, we have $\mathbb{E}_{P^{[t]}}[\theta_{i+1}^{t+1}-\theta_i^{t+1}]\geq K\eta_t> 0$.
\item $i=\tilde{y}^t-1$: In this case, $\tilde{\tau}_{\tilde{y}^t-1}^t=0$. Thus, 
\begin{align}
\label{eq:diff2}
\theta_{\tilde{y}^t}^{t+1} &- \theta_{\tilde{y}^t-1}^{t+1} 
= (1-\eta_t \lambda)( \theta_{\tilde{y}^t}^{t} - \theta_{\tilde{y}^t-1}^{t} ) -\eta_t \tilde{\tau}_{\tilde{y}^t}^t
\end{align}
We see that $\tilde{z}_{\tilde{y}^t}^t=-\frac{1}{P_t(\tilde{y}^t)}$ and $\tilde{\tau}_{\tilde{y}^t}^t=-\frac{1}{P_t(\tilde{y}^t)}\mathbb{I}[\frac{1}{P_t(\tilde{y}^t)}(\ww^t.\xx^t-\theta_{\tilde{y}^t})\leq 0]$. Thus, $\mathbb{E}_{P^t}[\tilde{\tau}^t_{\tilde{y}^t}]=\sum_{k=1}^KP^t(k)\frac{-1}{P^t(k)}\mathbb{I}[\frac{-1}{P^t(k)}(\ww^t.\xx^t-\theta_k^t)\leq 0]=-\sum_{k=1}^K \mathbb{I}[\frac{-1}{P^t(k)}(\ww^t.\xx^t-\theta_k^t)\leq 0]$. Taking expectations on both sides of eq.(\ref{eq:diff2}), we get
\begin{align*}
&\mathbb{E}_{P^{[t]}}[  \theta_{\tilde{y}^t}^{t+1} - \theta_{\tilde{y}^t-1}^{t+1} ]
= (1-\eta_t \lambda)\mathbb{E}_{P^{[t-1]}}[ \theta_{\tilde{y}^t}^{t} - \theta_{\tilde{y}^t-1}^{t} ]-\eta_t\mathbb{E}_{P^t}[\tilde{\tau}^t_{\tilde{y}^t}]\\
&\geq (1-\eta_t \lambda)\mathbb{E}_{P^{[t-1]}}[\theta_{\tilde{y}^t}^{t} - \theta_{\tilde{y}^t-1}^{t}]+\eta_t\sum_{k=1}^K \mathbb{I}[\frac{-1}{P^t(k)}(\ww^t.\xx^t-\theta_k^t)\leq 0]\\
&\geq K\eta_t+\eta_t\sum_{k=1}^K \mathbb{I}[\frac{-1}{P^t(k)}(\ww^t.\xx^t-\theta_k^t)\leq 0]>0
\end{align*} 
\item $i=\tilde{y}^t$: In this case $\tilde{\tau}^{t}_{\tilde{y}^t+1}=0$ and $\mathbb{E}_{P^t}[\tilde{\tau}^t_{\tilde{y}^t}]=-\sum_{k=1}^K\mathbb{I}[\frac{-1}{P^t(k)}(\ww^t.\xx^t-\theta_k^t)\leq 0]\geq -K$. Thus,
\begin{eqnarray*}
\theta_{\tilde{y}^t+1}^{t+1} - \theta_{\tilde{y}^t}^{t+1}  &=  (1-\eta_t\lambda)(\theta_{\tilde{y}^t+1}^t - \theta_{\tilde{y}^t}^t)+\eta_t\tilde{\tau}^t_{\tilde{y}^t} 
\end{eqnarray*}
Taking expectation on both sides, we get the following.
\begin{align*}
    \mathbb{E}_{P^{[t]}}[ \theta_{\tilde{y}^t+1}^{t+1} - \theta_{\tilde{y}^t}^{t+1} ] &=  (1-\eta_t\lambda)\mathbb{E}_{P^{[t-1]}}[(\theta_{\tilde{y}^t+1}^t - \theta_{\tilde{y}^t}^t)]+\eta_t\mathbb{E}_{P^t}[\tilde{\tau}^t_{\tilde{y}^t} ]\\
&\geq (1-\eta_t\lambda)\mathbb{E}_{P^{[t-1]}}[(\theta_{\tilde{y}^t+1}^t - \theta_{\tilde{y}^t}^t)]-K\eta_t\geq 0
\end{align*}
\item $i\in \{\tilde{y}^t+1,\ldots,K-1\}$: In this case, $\tilde{\tau}^t_{i+1}=\tilde{\tau}^t_i=0$. Thus,
\begin{align*}
    \theta^{t+1}_{i+1}-\theta_i^{t+1}=(1-\eta_t \lambda)(\theta_{i+1}^t-\theta_i^t)
\end{align*}
and taking expectation on both sides, we get
\begin{align*}
    \mathbb{E}_{P^{[t]}}[\theta^{t+1}_{i+1}-\theta_i^{t+1}]=(1-\eta_t \lambda)\mathbb{E}_{P^{[t-1]}}[(\theta_{i+1}^t-\theta_i^t)]\geq K\eta_t>0
\end{align*}
\end{enumerate}
\end{enumerate}

\subsection{Proof of Theorem 2}
The proof of Theorem 2 requires several other results. We will first present additional propositions required for the proof.
\begin{proposition}
\begin{enumerate}
        \item $ \Ex_{P^t}[(\taut_{\yt^t}^t)^2]  \leq \frac{6K^2}{\gamma}(1+\ln K)$
        \item $\Ex_{P^t}[\vert \taut^t_{\yt^t}\vert]\leq K$
        \end{enumerate}
    \label{lemma4}
\end{proposition}

\begin{proof}
\begin{enumerate}
    \item  Note that $(2d_t-1)^2=1$. Thus, we see that 
        \begin{align*} \Ex_{P^t}\left[\left(\taut_{\yt^t}^t\right)^2\right] 
        &= \Ex_{P^t}\left[\frac{(2d_t-1)^2}{(P^t(\tilde{y}^t))^2}\mathbb{I}[\frac{(2d_t-1)}{P_t(\tilde{y}^t)}(\ww^t.\xx^t-\theta_{\tilde{y}^t}^t)\leq 0]\right]\\
        &= \Ex_{P^t}\left[\frac{1}{(P^t(\tilde{y}^t))^2}\mathbb{I}[\frac{(2d_t-1)}{P_t(\tilde{y}^t)}(\ww^t.\xx^t-\theta_{\tilde{y}^t}^t)\leq 0]\right]
        \\
        &= \sum_{k=1}^K\frac{1}{P^t(k)}\mathbb{I}\left[\frac{2\mathbb{I}[k<y^t]-1}{P^t(k)}(\ww^t.\xx^t-\theta_{k}^t)\leq 0\right]\leq \sum_{k=1}^K\frac{1}{P^t(k)}
    \end{align*}
    We consider two cases here. (a) $2\hat{y}^t\geq K$ and (b) $2\hat{y}^t< K$
    \begin{enumerate}
        \item $2\hat{y}^t\geq K: $In this case:
        \begin{align*}
            \sum_{k=1}^K \frac{1}{P^t(k)}&=\frac{1}{(1-\gamma)+\gamma(\hat{y}^t+1)/Z^t}+\sum_{k=1}^{\hat{y}^t-1}\frac{Z^t}{\gamma(k+1)}+\sum_{k=\hat{y}^t+1}^K\frac{Z^t}{\gamma(2\hat{y}^t-k+1)}\\
            &\leq \frac{Z^t}{\gamma (\hat{y}^t+1)}+\sum_{k=1}^{\hat{y}^t-1}\frac{Z^t}{\gamma(i+1)}+\sum_{j=2\hat{y}^t-K}^{\hat{y}^t-1}\frac{Z^t}{\gamma(j+1)}\\
            &\leq  \sum_{k=1}^{\hat{y}^t}\frac{Z^t}{\gamma(i+1)}+\sum_{j=1}^{\hat{y}^t-1}\frac{Z^t}{\gamma(j+1)}\leq \sum_{j=1}^{K}\frac{2Z^t}{\gamma j}
        \end{align*}
        where $Z^t=(2K+1)\hat{y}^t-(\hat{y}^t)^2-\frac{K}{2}(K-1)$. $Z^t$ takes maximum value when $\hat{y}^t=\frac{2K+1}{2}$ which is $Z^t_{\max}=(K+\frac{1}{2})^2-\frac{K}{2}(K-1)\leq (K+\frac{1}{2})^2-(\frac{K}{2}-\frac{1}{2})^2=\frac{3K}{2}(\frac{K}{2}+1)<K(K+2)$. Also, $\sum_{k=1}^K\frac{1}{k}=\sum_{k=1}^K\int_{k-1}^k \frac{1}{k}dx <1+\sum_{k=2}^K\int_{k-1}^k\frac{1}{x}dx = 1+\int_1^K \frac{1}{x}dx = 1+\ln K$. Thus,
        \begin{align*}
            \sum_{k=1}^K\frac{1}{P^t(k)}\leq \frac{1}{\gamma} K(K+2)(1+\ln K)<\frac{2K^2}{\gamma}(1+\ln K).
        \end{align*}
        \item $2\hat{y}^t< K$: In this case, we see that
        \begin{align*}
            \sum_{k=1}^K\frac{1}{P^t(k)}&=\frac{1}{(1-\gamma)+\frac{\gamma(K-\hat{y}^t+1)}{Z^t}}+\sum_{k=1}^{\hat{y}^t-1}\frac{Z^t}{\gamma (K-2\hat{y}^t+k+1)}+\sum_{k=\hat{y}^t+1}^K\frac{Z^t}{\gamma (K+1-k)}\\
            &\leq \frac{Z^t}{\gamma}\sum_{j=K-2\hat{y}^t+2}^{K-\hat{y}^t+1}\frac{1}{j}+\frac{Z^t}{\gamma}\sum_{j=1}^{K-\hat{y}^t}\frac{1}{j}\leq \frac{2Z^t}{\gamma}\sum_{j=1}^{K}\frac{1}{j}\leq\frac{Z^t}{\gamma}(1+\ln K)
        \end{align*}
        where $Z^t= \frac{K(K+1)}{2}-\hat{y}^t(\hat{y}^t-1)$. But, we see that $Z^t\leq \frac{K(K+1)}{2}-\frac{1}{4}<\frac{K(K+1)}{2}<K^2$. Thus,
        \begin{align*}
            \sum_{k=1}^K\frac{1}{P^t(k)}\leq \frac{K^2}{\gamma}(1+\ln K).
        \end{align*}
    \end{enumerate}
    Combining the two cases, we conclude that
     \begin{align*}
            \sum_{k=1}^K\frac{1}{P^t(k)}\leq \frac{2K^2}{\gamma}(1+\ln K).
        \end{align*}
    \item Note that $\vert (2d_t-1)\vert=1$. Thus, we see that
    \begin{align*}
        \Ex_{P^t}[\vert \taut^t_{\yt^t}\vert]&=\Ex_{P^t}\left[\frac{1}{P^t(\tilde{y}^t)}\mathbb{I}[\frac{2d_t-1}{P^t(\tilde{y}^t)}(\ww^t.\xx^t-\theta_{\tilde{y}^t}^t)\leq 0]\right]\\
        &=\sum_{k=1}^KP^t(k)\frac{1}{P^t(k)}\mathbb{I}[\frac{2\mathbb{I}[k < y^t]-1}{P^t(k)}(\ww^t.\xx^t-\theta_k^t)\leq 0]\\
        &=\sum_{k=1}^K \mathbb{I}[\frac{2\mathbb{I}[k<y^t]-1}{P^t(k)}(\ww^t.\xx^t-\theta_k^t)\leq 0]\leq K.
    \end{align*}
    \end{enumerate}
\end{proof}

\begin{proposition}
\begin{enumerate}
\item $ \Ex_{P^{[t]}}[\| \uu^{t+1} \|] \leq \frac{ K\sqrt{R^2 + 1} }{\lambda},\;\forall t \in [T-1]$.
    \item $\Ex_{P^{[t]}}[\|\mathbf{u}^{t+1} \|^2]\leq \frac{(R^2+1)K^2}{\lambda^2}\left[\frac{6}{t\gamma}(1+\ln K)+2\right],\;\forall t \in [T-1]$
\end{enumerate}
        \label{lemma5}
\end{proposition}
\begin{proof}
   DFORD-Linear updates the parameters as follows.
    \begin{align*}
        \uu^{t+1} &= \uu^{t} - \eta_t \mathbf{g}(\mathbf{u}^t,\mathbf{x}^t,d^t) = \uu^{t} - \eta_t \lambda \uu^t - \eta_t[-( \sum_{i=1}^{K-1} \taut_i^t) \xx^t, \taut_1^t, \hdots, \taut_{K-1}^t]   
    \end{align*}
Let $\mathbf{v}^t=[(\sum_{i=1}^{K-1} \taut_i^t) \xx^t, -\taut_1^t, \hdots, -\taut_{K-1}^t]$. So, $\uu^{t+1} =  (1 - \eta_t \lambda) \uu^t + \eta_t \mathbf{v}^t$. Since $\uu^1 = \mathbf{0}$ and $\eta_t = \frac{1}{\lambda t},\; \forall t \in [T]$, we can easily show that
    \begin{equation*}
        \uu^{t+1} = \frac{1}{\lambda t} \sum_{s = 1}^{t}\mathbf{v}^s.
    \end{equation*}
\begin{enumerate}
    \item Taking norm on both sides, we get
    \begin{align*}
        \|\mathbf{u}^{t+1} \| &= \frac{1}{\lambda t} \| \sum_{s = 1}^{t}\mathbf{v}^s \| 
        \leq \frac{\sum_{s = 1}^{t}\| \mathbf{v}^s \|}{\lambda t}  
        \end{align*}
    By definition of $\taut_{i}^t$, we see that only $\taut_{\yt^t}^t$ will be non zero. Thus, $\mathbf{v}^s  =  [-( \taut_{\yt^s}^s) \xx^s, 0, \hdots, \taut_{\yt^s}^s, \hdots, 0] $.
  Hence, $\| \mathbf{v}^s \| = \| [-( \taut_{\yt^s}^s) \xx^s, 0, \hdots, \taut_{\yt^s}^s, \hdots, 0] \|= \vert \taut_{\yt^s}^s\vert .\sqrt{\|\xx^s\|^2 + 1}$. Taking expectation of $\|\mathbf{v}^s\|$, we get $\Ex_{P^{[t]}}[\|\mathbf{v}^s\|]=\sqrt{\|\xx^t\|^2+1}.\;\Ex_{P^t}[\vert \taut_{\yt^s}^s\vert]\leq K\sqrt{R^2+1}$. Thus, $\Ex_{P^{[t]}}[\|\mathbf{u}^{t+1}\|] \leq \frac{1}{\lambda}K\sqrt{R^2+1}$.
  \item Taking squared norm on both sides, we get
    \begin{align*}
        \|\mathbf{u}^{t+1} \|^2 &= \frac{1}{\lambda^2 t^2} \left(\sum_{s = 1}^{t}\|\mathbf{v}^s \|^2 + 2\sum_{i\neq j}\langle\mathbf{v}^i,\mathbf{v}^j\rangle \right)
        \leq \frac{1}{\lambda^2 t^2} \left(\sum_{s = 1}^{t}\|\mathbf{v}^s \|^2 + 2\sum_{i\neq j}\|\mathbf{v}^i\|.\|\mathbf{v}^j\| \right)\\
        &=\frac{1}{\lambda^2 t^2} \left(\sum_{s = 1}^{t}(\taut_{\yt^s}^s)^2 (\|\mathbf{x}^s\|^2+1)+ 2\sum_{i\neq j}\vert \taut_{\yt^i}^i\vert.\vert \taut_{\yt^j}^j \vert \sqrt{\|\mathbf{x}^i\|^2+1} \sqrt{\|\mathbf{x}^j\|^2+1} \right)\\
        &\leq\frac{1}{\lambda^2 t^2} \left((R^2+1)\sum_{s = 1}^{t}(\taut_{\yt^s}^s)^2 + 2(R^2+1)\sum_{i\neq j}\vert \taut_{\yt^i}^i\vert.\vert \taut_{\yt^j}^j\vert   \right)
        \end{align*}
Taking expectation on both sides with respect to $P^{[t]}$, we get 
\begin{align*}
    \Ex_{P^{[t]}}[\|\mathbf{u}^{t+1} \|^2]&\leq \frac{(R^2+1)}{\lambda^2 t^2}\Ex_{P^{[t]}}\left[\sum_{s = 1}^{t}(\taut_{\yt^s}^s)^2 + 2\sum_{i\neq j}\vert \taut_{\yt^i}^i\vert.\vert  \taut_{\yt^j}^j\vert  \right]\\
    &\leq \frac{(R^2+1)}{\lambda^2 t^2}\left[t\frac{2K^2}{\gamma}(1+\ln K)+2t(t-1)K^2\right]\leq \frac{(R^2+1)K^2}{\lambda^2}\left[\frac{2}{t\gamma}(1+\ln K)+2\right]
\end{align*}
  \end{enumerate}
\end{proof}

\begin{proposition}
    Let $\|\xx^t\| \leq R, \;\forall t \in [T]$. Then $\Ex_{P^{[t]}}[\|\mathbf{g}(\uu^{t}, \xx^t,d^t)\|^2]\leq 8(R^2 + 1)K^2\frac{(1+\ln K)}{\gamma},\;\forall t\in [T]$.
    \label{lemma6}
\end{proposition}
\begin{proof}
    \begin{align*} &\Ex_{P^{[t]}}\left[\left\|\mathbf{g}(\uu^{t};\xx^t,d^t)\right\|^2\right] =
     \Ex_{P^{[t]}}\left[\left\| \lambda \ww^t -\xx^t\sum_{i=1}^{K-1}\taut_i^t \right\|^2\right] +\Ex_{P^t}\left[\sum_{i=1}^{K-1}(\lambda \theta_i^t + \taut_i^t)^2\right] \\        
     &   = \Ex_{P^{[t]}}\left[\lambda^2 \| \ww^t\|^2 + \left(\sum_{i=1}^{K-1}\taut_i^t\right)^2\|\xx^t\|^2 + \lambda^2\sum_{i=1}^{K-1} (\theta_i^t)^2 + \left(\sum_{i=1}^{K-1}\taut_i^t\right)^2\right]   -2\lambda\Ex_{P^t}\left[\langle  \ww^t,\xx^t \rangle\sum_{i=1}^{K-1}\taut_i^t -\sum_{i=1}^{K-1} \theta_i^t \taut_i^t\right] 
        \\
       & = \Ex_{P^{[t]}}\left[\lambda^2 \| \ww^t\|^2 + \lambda^2\sum_{i=1}^{K-1} (\theta_i^t)^2\right] + (\|\xx^t\|^2+1)\Ex_{P^t}\left[(\sum_{i=1}^{K-1}\taut_i^t)^2 \right] + \Ex_{P^{[t]}}[-2\langle \lambda \ww^t,(\sum_{i=1}^{K-1}\taut_i^t)\xx^t \rangle + 2\sum_{i=1}^{K-1}\lambda \theta_i^t \tau_i^t]
        \\
        &=\Ex_{P^{[t-1]}}[(\lambda\| \uu^t)\|)^2] + (\|\xx^t\|^2 + 1)\Ex_{P^t}\left[(\sum_{i=1}^{K-1}\taut_i^t)^2\right] + 2\Ex_{P^{[t]}}[\langle \lambda \uu^t, [-(\sum_{i=1}^{K-1}\taut_i^t)\xx^t, \taut_1^t, \hdots, \taut_{K-1}^t]\rangle]\\ 
       & =  \lambda^2\Ex_{P^{[t-1]}}[\| \uu^t \|^2] + (\|\xx^t\|^2 + 1)\Ex_{P^t}\left[(\sum_{i=1}^{K-1}\taut_i^t)^2\right] + 2\Ex_{P^{[t-1]}}\langle \lambda \uu^t, [-\Ex_{P^t}[\taut^t_{\yt^t}]\xx^t, 0, \hdots,\Ex_{P^t}[\taut^t_{\yt^t}],\hdots, 0]\rangle 
   \end{align*}
   On applying Proposition~\ref{lemma5} to the first term, Proposition~\ref{lemma4} to the second term, and Cauchy Schwartz inequality to the third term, we get the following.
    \begin{align*}
       &\Ex_{P^{[t]}}[\|\mathbf{g}(\uu^{t};\xx^t,d^t)\|^2] 
       \leq K^2(R^2 + 1)\left[\frac{2}{t\gamma}(1+\ln K)+2\right]+ \frac{2K^2}{\gamma}(1+\ln K)(R^2 + 1) \\
       &   \;\;\;    +2\lambda \frac{K \sqrt{R^2 + 1}}{\lambda} \left\| \left[-\xx^t\Ex_{P^t}[\taut_{\yt^t}^t], 0, \hdots,\Ex_{P^t}[\taut_{\yt^t}^t],0,\hdots,0\right]\right\|
       \\
       &= (R^2 + 1)K^2\left(\frac{2(1+\ln K)}{t\gamma} + 2+\frac{2}{\gamma}(\ln K +1)\right) + 2K\sqrt{R^2 + 1}\sqrt{\|\xx^t\|^2+ 1}\;\vert\Ex_{P^t}[\taut^t_{\yt^t}]\vert\\
       &\leq (R^2 + 1)K^2\left(\frac{4(1+\ln K)}{\gamma} + 2\right) + 2K (R^2 + 1)\;\vert\Ex_{P^t}[\taut^t_{\yt^t}]\vert
       \end{align*}
       When $d^t=1$, we see that $\mathbb{E}_{P^t}[\tilde{\tau}^t_{\tilde{y}^t}]=\sum_{k=1}^KP^t(k)\frac{1}{P^t(k)}\mathbb{I}[\frac{1}{P^t(k)}(\ww^t.\xx^t-\theta_k^t)\leq 0]=\sum_{k=1}^K \mathbb{I}[\frac{1}{P^t(k)}(\ww^t.\xx^t-\theta_k^t)\leq 0]\leq K$. Similarly, when $d^t=0$, we see that $\mathbb{E}_{P^t}[\tilde{\tau}^t_{\tilde{y}^t}]=\sum_{k=1}^KP^t(k)\frac{-1}{P^t(k)}\mathbb{I}[\frac{-1}{P^t(k)}(\ww^t.\xx^t-\theta_k^t)\leq 0]=-\sum_{k=1}^K \mathbb{I}[\frac{-1}{P^t(k)}(\ww^t.\xx^t-\theta_k^t)\leq 0]\leq 0$. Thus, $\vert\Ex_{P^t}[\taut_{\yt^t}^t]\vert\leq K$. Using this, we get
       \begin{align*}
        \Ex_{P^{[t]}}[\|\mathbf{g}(\uu^{t};\xx^t,d^t)\|^2]
       &\leq 2(R^2 + 1)K^2\left(\frac{2(1+\ln K)}{\gamma} + 1\right) + 2K^2(R^2 + 1)\\
       &= 2(R^2 + 1)K^2\left(\frac{2(1+\ln K)}{\gamma} + 2\right)\leq 8(R^2 + 1)K^2\frac{(1+\ln K)}{\gamma}.
    \end{align*}
    
\end{proof}

\paragraph{\bf Proof of Theorem 2}

\begin{proof}
    \begin{align*}
        \|\uu^{t} - \uu \|^2 - \|\uu^{t+1} - \uu \|^2
        &= \|\uu^t - \uu \|^2 - \| \uu^t - \eta_t \mathbf{g}(\uu^{t},\xx^t,d^t) - \uu\|^2 \\
        &= 2\eta_t\langle \uu^{t} - \uu, \mathbf{g}(\uu^{t},\xx^t,d^t) \rangle - \eta_t^2\| \mathbf{g}(\uu^{t},\xx^t,d^t)\|^2 
    \end{align*}
    Let us use the notation $P^{[t]}$ for $P^1P^2 \ldots P^t$ which is joint distribution of $\tilde{y}^1,\ldots,\tilde{y}^t$. Taking expectation with respect to $P^{[t]}$ on both sides, we get
    \begin{align*}
       \Ex_{P^{[t]}}[ \|\uu^{t} - \uu \|^2 - \|\uu^{t+1} - \uu \|^2] 
       &= 2\eta_t \Ex_{P^{[t]}}[\langle \uu^{t} - \uu, \mathbf{g}(\uu^{t},\xx^t,d^t) \rangle] - \eta_t^2\Ex_{P^{[t]}}[\|\mathbf{g}(\uu^{t},\xx^t,d^t)\|^2] \\ 
       & = 2\eta_t \Ex_{P^{[t-1]}}[\langle \uu^{t} - \uu, \Ex_{P^t}[\mathbf{g}(\uu^{t},\xx^t,d^t)] \rangle] - \eta_t^2\Ex_{P^{[t]}}[\| \mathbf{g}(\uu^{t},\xx^t,d^t)\|^2] \\
       & = 2\eta_t \Ex_{P^{[t-1]}}\langle \uu^{t} - \uu, \nabla L_{reg}^{IMC}(\uu^t,\xx^t,y^t) \rangle - \eta_t^2\Ex_{P^{[t]}}[\|\mathbf{g}(\uu^{t},\xx^t,d^t)\|^2]. 
    \end{align*}
In the above, we used the fact that $\Ex_{P^t}[\mathbf{g}(\uu^{t},\xx^t,d^t)]=L_{reg}^{IMC}(\uu^t,\xx^t,y^t)$. On rearranging the terms we get,
    \begin{align}
        \Ex_{P^{[t-1]}}\left[\langle \uu^{t} - \uu, \nabla L_{reg}^{IMC}(\uu^t,\xx^t,y^t) \rangle\right] = \frac{1}{2\eta_t}\Ex_{P^{[t]}}[ \|\uu^{t} - \uu \|^2 - \|\uu^{t+1} - \uu \|^2] + \frac{\eta_t}{2}\Ex_{P^{[t]}}[\|\mathbf{g}(\uu^{t},\xx^t,d^t)\|^2]
        \label{eqn:regret1}
    \end{align}
Using the fact that $L_{reg}^{IMC}(\uu,\xx^t,y^t)$ is a $\lambda$-strongly convex function of $\uu$, we have
\begin{align}
\label{eq:strong-convexity-Lreg}
  L_{reg}^{IMC}(\uu,\xx^t,y^t)-L_{reg}^{IMC}(\uu^t,\xx^t,y^t)  \geq (\uu-\uu^t)^\top\nabla L_{reg}^{IMC}(\uu^t,\xx^t,y^t) + \frac{\lambda}{2}\Vert \uu-\uu^t\Vert^2.
\end{align}
    Using eq.(\ref{eqn:regret1}) and (\ref{eq:strong-convexity-Lreg}), we get the following inequality. 
    \begin{align*}
        \Ex_{P^{[t]}}[L_{reg}^{IMC}(\uu^t,\xx^t,y^t) - L_{reg}^{IMC}(\uu,\xx^t,y^t)] 
         &\leq \Ex_{P^{[t]}}[ \frac{\|\uu^{t} - \uu \|^2 - \|\uu^{t+1} - \uu \|^2}{2\eta_t} - \frac{\lambda}{2}\|\uu^t-\uu\|^2] \\
         &\;\;+ \frac{\eta_t}{2}\Ex_{P^{[t]}}[\|\mathbf{g}(\uu^{t},\xx^t,d^t)\|^2]
    \end{align*}
From Proposition \ref{lemma6} we get
    \begin{align*}
        &\Ex_{P^{[t]}}[L_{reg}^{IMC}(\uu^t,\xx^t,y^t) - L_{reg}^{IMC}(\uu,\xx^t,y^t)] \\
        &\leq \Ex_{P^{[t]}}[ \frac{\|\uu^{t} - \uu \|^2 - \|\uu^{t+1} - \uu \|^2}{2\eta_t} - \frac{\lambda}{2}\|\uu^t-\uu\|^2] + 4\eta_t(R^2 + 1)K^2\frac{(1+\ln K)}{\gamma}.
    \end{align*}
    On summing up from $t = 1$ to $ T$, we get
    \begin{align*}
        \sum_{t =1}^{T} \Ex_{P^{[t]}}[L_{reg}^{IMC}(\uu^t,\xx^t,y^t) - L_{reg}^{IMC}(\uu,\xx^t,y^t)]&\leq \sum_{t=1}^T\Ex_{P^{[t]}}[(\frac{\|\uu^{t} - \uu \|^2 - \|\uu^{t+1} - \uu \|^2}{2\eta_t}\\
        &- \frac{\lambda}{2}\|\uu^t-\uu\|^2)] 
        + 4(R^2 + 1)K^2\frac{(1+\ln K)}{\gamma}\sum_{t =1}^{T}\eta_t\\
       \Rightarrow \Ex_{P^{[T]}}\left[\sum_{t =1}^{T} \left(L_{reg}^{IMC}(\uu^t,\xx^t,y^t) - L_{reg}^{IMC}(\uu,\xx^t,y^t)\right)\right]&\leq \Ex_{P^{[T]}}\left[\sum_{t=1}^T(\frac{\|\uu^{t} - \uu \|^2 - \|\uu^{t+1} - \uu \|^2}{2\eta_t} - \frac{\lambda}{2}\|\uu^t-\uu\|^2)\right] 
        \\ 
       & \;\;+4(R^2 + 1)K^2\frac{(1+\ln K)}{\gamma}\sum_{t =1}^{T}\eta_t
    \end{align*}
    On substituting $\eta_t = \frac{1}{\lambda t}$, we get
    \begin{align*}
       & \Ex_{P^{[T]}}\left[\sum_{t =1}^{T} \left(L_{reg}^{IMC}(\uu^t,\xx^t,y^t) - L_{reg}^{IMC}(\uu,\xx^t,y^t)\right)\right]\\
       &\leq -\lambda (T + 1)\Ex_{P^{[T]}}[\|\uu^{T+1} - \uu\|^2] + \frac{4}{\lambda}(R^2 + 1)K^2\frac{(1+\ln K)}{\gamma}\sum_{t =1}^{T}\frac{1}{t}
        \\
        & \leq \frac{4}{\lambda\gamma}(R^2 + 1)K^2(1+\ln K) (1 + \ln{T} )\\
        &\leq \frac{16}{\lambda\gamma}(R^2 + 1)K^2\ln K \ln{T} 
    \end{align*}
    \end{proof}

\subsection{Proof  of Lemma 2}

\begin{proof}
        \begin{align*}
        f^{t} &= (1 - \eta_{t-1} \lambda)f^{t-1} + \eta_{t-1} \taut^{t-1}_{\yt^{t-1}} k(\xx^{t-1}, .) 
        \\
        &= (1 - \eta_{t-1} \lambda)\left((1 - \eta_{t-2} \lambda)f^{t-2} + \eta_{t-2} \taut^{t-2}_{\yt^{t-2}} k(\xx^{t-2}, .)\right)
        + \eta_{t-1} \taut^{t-1}_{\yt^{t-1}} k(\xx^{t-1}, .)
        \\
        &\;\;\vdots
        \\
        f^t &= f_1\prod_{j=1}^{t-1} (1 - \eta_{t-j}\lambda) + \sum_{p=1}^{t-2} \left(\prod_{j=1}^{p} (1 - \eta_{t-j} \lambda)\right)\eta_{t-p-1} \taut^{t-p-1}_{\yt^{t-p-1}} k(\xx^{t-p-1}, .)
        + \eta_{t-1} \taut^{t-1}_{\yt^{t-1}} k(\xx^{t-1}, .)
        \\
        &= \sum_{p=1}^{t-2} \left(\prod_{j=1}^{p} (1 - \eta_{t-j} \lambda)\right)\eta_{t-p-1} \taut^{t-p-1}_{\yt^{t-p-1}} k(\xx^{t-p-1}, .)
        + \eta_{t-1} \taut^{t-1}_{\yt^{t-1}} k(\xx^{t-1}, .)
    \end{align*}
But, we see that
\begin{align*}
        \left(\prod_{j=1}^{p} (1 - \eta_{t-j} \lambda)\right)\eta_{t-p-1}= (1 - \frac{\lambda}{\lambda (t-1)})\times(1 - \frac{\lambda}{\lambda (t-2)}) \hdots (1 - \frac{\lambda}{\lambda (p+2)})\times(1 - \frac{\lambda}{\lambda (p+1)}) \frac{1}{\lambda p
        }
        \\
        =(\frac{t-2}{t-1}) \times (\frac{t-3}{t-2}) \hdots (\frac{p-1}{p+2}) \times (\frac{p}{p+1}) \times \frac{1}{\lambda p} = \frac{1}{\lambda (t-1)}.
    \end{align*}
Substituting this back, we get
    \begin{align*}
        f^t &=\sum_{p=1}^{t-2} \frac{1}{\lambda(t-1)} \taut^{t-p-1}_{\yt^{t-p-1}} k(\xx^{t-p-1}, .)
        + \eta_{t-1} \taut^{t-1}_{\yt^{t-1}} k(\xx^{t-1}, .)
        = \sum_{p=0}^{t-2}\frac{1}{\lambda (t-1)} \taut^{t-p-1}_{\yt^{t-p-1}} k(\xx^{t-p-1}, .)\\
        &= \frac{1}{\lambda (t-1)}\sum_{s=1}^{t-1} \taut^{s}_{\yt^{s}} k(\xx^{s}, .)
    \end{align*}
\end{proof}

\subsection{Proof of Theorem 3}
\begin{proposition} Let $k(\xx, \xx) \leq R_1^2$, then    
\begin{enumerate}
    \item $ \Ex_{P^{[t-1]}}[\| f^t \|] \leq \frac{KR_1}{\lambda} $
    \item $ \Ex_{P^{[t-1]}}[\| f^t \|^2]\leq  \frac{2R_1^2K^2}{\lambda^2} \left(\frac{1}{\gamma (t-1)}(1+\ln K)+1\right)$.
\end{enumerate}
\label{ker1}
\end{proposition}

\begin{proof}
    From lemma $4.4$ we know that
    $f^{t}(.) = \frac{1}{\lambda (t-1)}\sum_{s = 1}^{t-1} \taut^s_{\yt^s} k(\xx^s,.)$
   Since we are using the truncation technique, $f^{t+1}$ will look like
    \begin{align*}
    f^{t}(.) &= \frac{1}{\lambda (t-1)}\sum_{s = \max(1, t - \delta - 1)}^{t-1} \taut^s_{\yt^s} k(\xx^s,.)  \text{, 
  where } \delta \text{ is the truncation parameter} 
    \end{align*}
    \begin{align*}
        \| f^t \| &\leq \frac{1}{\lambda (t-1)} \sum_{s= \max(1, t - \delta - 1)}^{t-1} \|\taut^{s}_{\yt^s} k(\xx^s, .)\|
        \leq \frac{1}{\lambda (t-1)} \sum_{s= \max(1, t - \delta - 1)}^{t-1} \vert \taut^{s}_{\yt^s}\vert R_1.
    \end{align*}

    Let $\mu = \max(1, t - \delta - 1)$
    \begin{enumerate}

        \item Taking expectation of $\| f^t \|$ with respect to $P^{[t-1]}$, we get
    \begin{align*}
        \Ex_{P^{[t-1]}}[\| f^t \|]
        &\leq \frac{R_1}{\lambda (t-1)}\sum_{s=\mu}^{t-1}\Ex_{P^{[t-1]}} \left[\vert \taut^{s}_{\yt^s}\vert\right]=\frac{R_1}{\lambda (t-1)}\sum_{s=\mu}^{t-1}\Ex_{P^s} \left[\vert \taut^{s}_{\yt^s}\vert\right]
        \\ &\leq
        \frac{R_1}{\lambda (t-1)}\sum_{s=\mu}^{t-1}K= \frac{(t - 1 -\mu)KR_1}{\lambda(t-1)}
        \leq \frac{KR_1}{\lambda}.
    \end{align*}
    
    \item Taking expectation of $\| f^t \|^2$, we get 
    \begin{align*}
        &\Ex_{P^{[t-1]}}[\| f^t \|^2]
        \leq \Ex_{P^{[t-1]}} \left[ \frac{R_1^2}{\lambda^2 (t-1)^2}(\sum_{s=\mu}^{t-1}\vert \taut^{s}_{\yt^s}\vert)^2\right]
        \\ &=
         \frac{R_1^2}{\lambda^2 (t-1)^2}
         \Ex_{P^{[t-1]}} \left[\sum_{s=\mu}^{t-1}\vert \taut^{s}_{\yt^s}\vert^2
         + \sum_{i,j \in [\mu, t-1],\atop i \neq j}2\vert \taut^{i}_{\yt^i}\vert \vert \taut^{j}_{\yt^j}\vert
         \right]
        \\ 
        &=
         \frac{R_1^2}{\lambda^2 (t-1)^2}
         \left(\sum_{s=\mu}^{t-1}\Ex_{P^{[t-1]}} \left[\vert \taut^{s}_{\yt^s}\vert^2 \right]
         +  \sum_{i,j \in [\mu, t-1],\atop i \neq j}2\Ex_{P^{[t-1]}} \left[\vert \taut^{i}_{\yt^i}\vert\right]\Ex_{P^{[t-1]}} \left[ \vert \taut^{j}_{\yt^j}\vert
         \right]\right) 
         \\ &\leq
        \frac{R_1^2}{\lambda^2 (t-1)^2}
        \left(\sum_{s=\mu}^{t-1}\frac{2K^2}{\gamma}(1+\ln K) + \sum_{i,j \in [\mu, t-1],\atop i \neq j}2K^2\right) \textit{ \hspace{0.8in} (from \ref{lemma4} and Theorem 4.1)}
        \\ &=
        \frac{R_1^2}{\lambda^2 (t-1)^2} \left((t-1 - \mu)\frac{2K^2}{\gamma}(1+\ln K)+2K^2(t-1-\mu)(t-2-\mu)\right)\\
        &\leq  \frac{2R_1^2K^2}{\lambda^2} \left(\frac{1}{\gamma (t-1)}(1+\ln K)+1\right)
    \end{align*}
    \end{enumerate}
\end{proof}

\begin{proposition}
    Let $\thetaa^t$ from DFORD-Kernel, then
    \begin{enumerate}
    \item $\Ex_{P^{[t]}}[\| \thetaa^{t+1} \|] \leq \frac{K}{\lambda}, \forall t \in [T - 1].$
    \item $\Ex_{P^{[t]}}[\| \thetaa^{t+1} \|^2]\leq\frac{2K^2}{\lambda^2}\left(\frac{1}{t\gamma}(1+\ln K)+1\right), \forall t \in [T - 1].$
    \end{enumerate}
\label{ker2}
\end{proposition}
\begin{proof}
   DFORD-Kernel updates $\thetaa$ as follows.
    \begin{align*}
        \thetaa^{t+1} &= (1-\eta_t\lambda)\thetaa^{t} -\eta_t[\taut_1^t, \hdots, \taut_{K-1}^t]   
    \end{align*}
Let $\mathbf{v}^t=[-\taut_1^t, \hdots, -\taut_{K-1}^t]$. So, $\thetaa^{t+1} =  (1 - \eta_t \lambda) \thetaa^t + \eta_t \mathbf{v}^t$. Since $\thetaa^1 = \mathbf{0}$ and $\eta_t = \frac{1}{\lambda t},\; \forall t \in [T]$, we can easily show that
    \begin{equation*}
        \thetaa^{t+1} = \frac{1}{\lambda t} \sum_{s = 1}^{t}\mathbf{v}^s.
    \end{equation*}
    Taking norm on both sides, we get
    \begin{align*}
        \|\mathbf{\thetaa}^{t+1} \| &= \frac{1}{\lambda t} \| \sum_{s = 1}^{t}\mathbf{v}^s \| 
        \leq \frac{\sum_{s = 1}^{t}\| \mathbf{v}^s \|}{\lambda t}  
        \end{align*}
    By definition of $\taut_{i}^t$, we see that only $\taut_{\yt^t}^t$ will be non zero. Thus, $\mathbf{v}^s  =  [ 0, \hdots, \taut_{\yt^s}^s, \hdots, 0] $.
  Hence, $\| \mathbf{v}^s \| = \| [ 0, \hdots, \taut_{\yt^s}^s, \hdots, 0] \|= \vert \taut_{\yt^s}^s\vert$. Taking expectation of $\|\mathbf{v}^s\|$, we get $\Ex_{P^{[t]}}[[\|\mathbf{v}^s\|]=\Ex_{P^t}[\vert \taut^s_{\yt^s}\vert]\leq K$. Thus, $\Ex_{P^{[t]}}[\|\mathbf{\thetaa}^{t+1}\|] \leq \frac{K}{\lambda}$. Now, taking the expectation of the square of $\Vert \thetaa^{t+1}\Vert$, we get
  \begin{align*}
      &\Ex_{P^{[t]}}[\Vert \thetaa^{t+1}\Vert]\leq \frac{1}{\lambda^2t^2}\Ex_{P^{[t]}}\left[(\sum_{s=1}^t\Vert \mathbf{v}^s\Vert)^2\right]=\frac{1}{\lambda^2t^2}\Ex_{P^{[t]}}\left[(\sum_{s=1}^t\vert \taut^s_{\yt^s})^2\right]\\
      &=\frac{1}{\lambda^2t^2}\Ex_{P^{[t]}}\left[\sum_{s=1}^t (\taut^s_{\yt^s})^2 +2\sum_{i\neq j}\vert \taut_i^{\yt^i}\vert.\vert \taut^j_{\yt^j}\vert\right]=\frac{1}{\lambda^2t^2}\left[\sum_{s=1}^t\Ex_{P^{s}}[(\taut^s_{\yt^s})^2]+2\sum_{i\neq j}\Ex_{P^i}[\vert \taut_i^{\yt^i}\vert].\Ex_{P^j}\vert \taut^j_{\yt^j}\vert]\right]\\
      &\leq \frac{1}{\lambda^2t^2}\left(\frac{2tK^2}{\gamma}(1+\ln K)+2t(t-1)K^2\right)
      \leq\frac{2K^2}{\lambda^2}\left(\frac{1}{t\gamma}(1+\ln K)+1\right)
  \end{align*}
\end{proof}

\begin{proposition}
\label{lemma:grad-norm-bound-kern}
    Let $k(\xx, \xx) \leq R_1^2$. Then $\Ex_{P^{[t]}} [\| \mathbf{g}((f^t, \thetaa^t),\xx^t, d^t) \|^2] \leq \frac{8}{\gamma}(R_1^2+1)K^2(1+\ln K),\;\forall t\in [T].$
\end{proposition}

\begin{proof}
    \begin{align*}
        &\| \mathbf{g}((f^t,\thetaa^t),\xx^t,d^t)\|^2 
        = \left\Vert \begin{bmatrix}
        \lambda f^t-\sum_{i=1}^{K}\taut_i^t k(\xx^t,.)\\
        \lambda \theta_1^t+\taut_1^t\\
        \vdots\\
        \lambda \theta_{K-1}^t+\taut_{K-1}^t
        \end{bmatrix}
        \right\Vert^2=\left\Vert \begin{bmatrix}
        \lambda f^t-\taut_{\yt^t}^t k(\xx^t,.)\\
        \lambda \theta_1^t\\
        \vdots\\
        \lambda \theta_{\yt^t}^t+\taut^t_{\yt^t}\\
        \vdots\\
        \lambda \theta_{K-1}^t
        \end{bmatrix}\right\Vert^2\\
        &= \lambda^2 \| f^t \|^2 + \| \taut_{\yt^t}^tk(\xx^t, . ) \|^2 -2 \langle \lambda f^t,\taut_{\yt^t}^tk(\xx^t, .) \rangle 
        + \lambda^2 \| \thetaa^t \|^2 + (\taut^t_{\yt^t} )^2 
        + 2 \lambda \theta^t_{\yt^t}\taut^t_{\yt^t}\\
        &\leq \lambda^2 \| f^t \|^2 + \vert \taut_{\yt^t}^t\vert^2\| k(\xx^t, . ) \|^2 +2  \lambda\Vert f^t\Vert.\Vert k(\xx^t, .)\Vert.\vert \taut_{\yt^t}^t\vert 
        + \lambda^2 \| \thetaa^t \|^2 + (\taut^t_{\yt^t} )^2 
        + 2 \lambda \Vert \thetaa^t\Vert.\vert \taut^t_{\yt^t}\vert 
    \end{align*}
    Taking expectation with respect to $P^{[t]}$ on both sides, we get
    \begin{align*}
    &\Ex_{P^{[t]}}\left[\Vert\mathbf{g}((f^t,\thetaa^t),\xx^t,d^t)\Vert^2\right]\leq \lambda^2 \Ex_{P^{[t]}}\left[\| f^t \|^2\right] + \Ex_{P^{[t]}}\left[\vert \taut_{\yt^t}^t\vert^2\| k(\xx^t, . ) \|^2\right] \\
    &+2  \lambda\Ex_{P^{[t]}}\left[\Vert f^t\Vert.\Vert k(\xx^t, .)\Vert.\vert \taut_{\yt^t}^t\vert \right]
        + \lambda^2 \Ex_{P^{[t]}}\left[\| \thetaa^t \|^2 \right]+ \Ex_{P^{[t]}}\left[(\taut^t_{\yt^t} )^2\right] 
        + 2 \lambda \Ex_{P^{[t]}}\left[\Vert \thetaa^t\Vert.\vert\taut^t_{\yt^t}\vert\right]\\
        &\leq \lambda^2 \Ex_{P^{[t-1]}}\left[\| f^t \|^2\right] + R_1^2\Ex_{P^{t}}\left[\vert \taut_{\yt^t}^t\vert^2\ \right] +2  \lambda R_1\Ex_{P^{[t-1]}}\left[\Vert f^t\Vert\right].\Ex_{P^t}\left[\vert \taut_{\yt^t}^t\vert \right]
        + \lambda^2 \Ex_{P^{[t-1]}}\left[\| \thetaa^t \|^2 \right]\\
        &\;\;\;\;+ \Ex_{P^{t}}\left[(\taut^t_{\yt^t} )^2\right] 
        + 2 \lambda \Ex_{P^{[t-1]}}\left[\Vert\thetaa^t\Vert\right]\Ex_{P^t}\left[\vert\taut^t_{\yt^t}\vert\right].
    \end{align*}
    Using Propositions~\ref{ker1} and \ref{ker2}, we get 

    \begin{align*}   
    \Ex_{P^{[t]}}\left[\Vert \mathbf{g}((f^t,\thetaa^t),\xx^t,d^t)\Vert^2\right]&\leq 2R_1^2K^2 \left(\frac{1}{\gamma (t-1)}(1+\ln K)+1\right) + 2K^2 R_1^2\frac{(1+\ln K)}{\gamma} \\
    &+ 2K^2R_1^2 
    + 2K^2\left(\frac{1}{(t-1)\gamma}(1+\ln K)+1\right)+ \frac{2K^2}{\gamma}(1+\ln K) + 2K^2\\
    &\leq 4R_1^2K^2\left(\frac{1}{\gamma}(1+\ln K)+1\right)+4K^2\left(\frac{1}{\gamma}(1+\ln K)+1\right)\\
    &= 4(R_1^2+1)K^2\left(\frac{1}{\gamma}(1+\ln K)+1\right)\\
    &\leq \frac{8}{\gamma}(R_1^2+1)K^2(1+\ln K)
    \end{align*}
    
\end{proof}

\paragraph{\bf Proof of Theorem 3}

\begin{proof}
    \begin{align*}
       & \|(f^t, \thetaa^t)- (f, \thetaa)\|^2 - \| (f^{t+1}, \thetaa^{t+1}) - (f, \thetaa) \|^2 \\
       &= \| (f^{t+1}, \thetaa^{t+1}) - (f^t, \thetaa^t) \|^2 
        - 2\langle (f^{t+1}, \thetaa^{t+1}) - (f, \thetaa) , (f^t, \thetaa^t)- (f, \thetaa) \rangle\\
        &= -\eta_t^2 \| \mathbf{g}((f^t,\thetaa^t),\xx^t, d^t) \|^2
        + 2\eta_t \langle \mathbf{g}((f^t,\thetaa^t),\xx^t, d^t) , (f^t, \thetaa^t)- (f, \thetaa) \rangle
    \end{align*}     
    Taking Expectation on both sides w.r.t the random variable $P^{[t]}$.
    \begin{align*}
        \Ex_{P^{[t]}}[ \|(f^t, \thetaa^t)&- (f, \thetaa)\|^2 - \| (f^{t+1}, \thetaa^{t+1}) - (f, \thetaa) \|^2 ]
        \\
        &= -\eta_t^2 \Ex_{P^{[t]}}[ \| \mathbf{g}((f^t,\thetaa^t),\xx^t, d^t) \|^2 ]
        + 2\eta_t \Ex_{P^{[t]}}[ \langle \mathbf{g}((f^t,\thetaa^t),\xx^t, d^t) , (f^t, \thetaa^t)- (f, \thetaa) \rangle ]\\
        &= -\eta_t^2 \Ex_{P^{[t]}}[ \| \mathbf{g}((f^t,\thetaa^t),\xx^t, d^t) \|^2 ]
        + 2\eta_t \Ex_{P^{[t-1]}}[ \langle \Ex_{P^t}[\mathbf{g}((f^t,\thetaa^t),\xx^t, d^t)] , (f^t, \thetaa^t)- (f, \thetaa) \rangle ]\\
        &= -\eta_t^2 \Ex_{P^{[t]}}[ \| \mathbf{g}((f^t,\thetaa^t),\xx^t, d^t) \|^2 ]
        + 2\eta_t \Ex_{P^{[t-1]}}[ \langle \nabla L_{reg}^{IMC}((f^t,\thetaa^t),\xx^t, y^t), (f^t, \thetaa^t)- (f, \thetaa) \rangle ]
    \end{align*}
    where we have used the fact that $\Ex_{P^t}[\mathbf{g}(f^t, \thetaa^t,\xx^t, d^t)]=\nabla L_{reg}^{IMC}(f^t, \thetaa^t,\xx^t, y^t)]$. On rearranging this equation, we get
    \begin{multline}
        \Ex_{P^{[t-1]}}[\langle  \nabla L_{reg}^{IMC}((f^t, \thetaa^t),\xx^t, y^t)) , (f^t, \thetaa^t)- (f, \thetaa) \rangle] 
        \\
        = \frac{1}{2 \eta_t}\Ex_{P^{[t]}}[ \|(f^t, \thetaa^t)- (f, \thetaa)\|^2 - \| (f^{t+1}, \thetaa^{t+1}) - (f, \thetaa) \|^2] + \frac{\eta_t}{2} \Ex_{P^{[t]}}[\| \mathbf{g}(f^t, \thetaa^t,\xx^t, d^t) \|^2]
        \label{eq:kern-grad-bound}
    \end{multline}
    Using Strong convexity of $L_{reg}^{IMC}$, we know that
    \begin{align}
         \nonumber &L_{reg}^{IMC}((f, \thetaa),\xx^t, y^t)-L_{reg}^{IMC}((f^t, \thetaa^t),\xx^t, y^t) \geq  \frac{\lambda}{2} \|(f^t, \thetaa^t)- (f, \thetaa)\|^2 \\
         &\quad + \langle  \nabla L_{reg}^{IMC}((f^t, \thetaa^t),\xx^t, y^t) ,  (f, \thetaa)-(f^t, \thetaa^t) \rangle \label{eq:kern-strong-conv}
         \end{align}
   Using eq.(\ref{eq:kern-grad-bound}) and (\ref{eq:kern-strong-conv}), we get the following.
   \begin{align*}
        &\Ex_{P^{[t-1]}}[L_{reg}^{IMC}((f^t, \thetaa^t),\xx^t, y^t) -L_{reg}^{IMC}((f, \thetaa),\xx^t, y^t)+ \frac{\lambda}{2} \|(f^t, \thetaa^t)- (f, \thetaa)\|^2]\\
        &\leq \frac{1}{2 \eta_t}\Ex_{P^{[t]}}[\|(f^t, \thetaa^t)- (f, \thetaa)\|^2 - \| (f^{t+1}, \thetaa^{t+1}) - (f, \thetaa) \|^2] + \frac{\eta_t}{2} \Ex_{P^{[t]}}[\| \mathbf{g}((f^t, \thetaa^t),\xx^t, d^t) \|^2]
    \end{align*}
    Using Proposition~\ref{lemma:grad-norm-bound-kern}, we get the following.
    \begin{align*}
        &\Ex_{P^{[t-1]}}[L_{reg}^{IMC}((f^t, \thetaa^t),\xx^t, y^t) -L_{reg}^{IMC}((f, \thetaa),\xx^t, y^t)+ \frac{\lambda}{2} \|(f^t, \thetaa^t)- (f, \thetaa)\|^2]\\
        &\leq \frac{1}{2 \eta_t}\Ex_{P^{[t]}}[\|(f^t, \thetaa^t)- (f, \thetaa)\|^2 - \| (f^{t+1}, \thetaa^{t+1}) - (f, \thetaa) \|^2] +  \frac{4\eta_t}{\gamma}(R_1^2+1)K^2(1+\ln K)
    \end{align*}
    Summing from $t=1$ to $T$, we get
    \begin{align*}
        &\sum_{t=1}^T\Ex_{P^{[t-1]}}[L_{reg}^{IMC}((f^t, \thetaa^t),\xx^t, y^t) -L_{reg}^{IMC}((f, \thetaa),\xx^t, y^t)] \\
        &\leq \sum_{t=1}^T\Ex_{P^{[t]}}\left[\frac{\|(f^t, \thetaa^t)- (f, \thetaa)\|^2 - \| (f^{t+1}, \thetaa^{t+1}) - (f, \thetaa) \|^2}{2\eta_t}-\frac{\lambda}{2} \|(f^t, \thetaa^t)- (f, \thetaa)\|^2\right] \\
        &+  \frac{4}{\gamma}(R_1^2+1)K^2(1+\ln K)\sum_{t=1}^T\eta_t
    \end{align*}
    We can rewrite it as follows.
    \begin{align*}
&\Ex_{P^{[T]}}\left[\sum_{t=1}^T\left[L_{reg}^{IMC}((f^t, \thetaa^t),\xx^t, y^t) -L_{reg}^{IMC}((f, \thetaa),\xx^t, y^t)\right]\right] \\
        &\leq \Ex_{P^{[T]}}\left[\sum_{t=1}^T\left[\frac{\|(f^t, \thetaa^t)- (f, \thetaa)\|^2 - \| (f^{t+1}, \thetaa^{t+1}) - (f, \thetaa) \|^2}{2\eta_t}-\frac{\lambda}{2} \|(f^t, \thetaa^t)- (f, \thetaa)\|^2\right]\right] \\
        &+  \frac{4}{\gamma\lambda}(R_1^2+1)K^2(1+\ln K)\sum_{t=1}^T\frac{1}{t}\\
        &\leq -\lambda (T+1)\Ex_{P^{[T]}}\left[\|(f^{T+1}, \thetaa^{T+1})- (f, \thetaa)\|^2 \right]+\frac{4}{\gamma\lambda}(R_1^2+1)K^2(1+\ln K)\\
        &\leq \frac{4}{\gamma\lambda}(R_1^2+1)K^2(1+\ln K)(1+\ln T)\leq \frac{16}{\gamma\lambda}(R_1^2+1)K^2\ln K\ln T
    \end{align*}
  
\end{proof}

\section{Baseline Algorithms Used}
\subsection{PRank (Full Information Based Online Ordinal Regression)}
\label{sec:PRank}
We discuss the algorithm for linear classifiers. Which means, $f(\xx)=\ww\cdot\xx$. Thus, the parameters to be estimated are $\ww$ and $\thetaa$. We initialize with $\ww^0=\mathbf{0}$ and $\thetaa^0=\mathbf{0}$. Let $\ww^t,\thetaa^t$ be the estimates of the parameters at the beginning of trial $t$. Let at trial $t$, $\xx^t$ be the example observed, and $y^t$ be its label.
At round $t$, the algorithm performs stochastic gradient descent on $L^{H}_{reg}(f,\thetaa,\xx^t,y^t)$ to update the parameters. 
$\ww^{t+1}$ and $\thetaa^{t+1}$ are found as follows.
\begin{align*}
\ww^{t+1} &= \ww^t - \eta_t\nabla_{\ww} L^{H}_{reg}(\ww^t.\xx^t,\thetaa^t,\xx^t,y^t)=(1-\eta_t\lambda)\ww^t +\eta_t  \sum_{i\in [K-1]}\tau_i^t\xx^t \\
\theta^{t+1}_i &= \theta^t_i - \eta_t\frac{\partial  L^{H}_{reg}(\ww^t.\xx^t,\thetaa^t,\xx^t,y^t)}{\partial \theta_i}= (1-\eta_t\lambda)\theta_i^t - \eta_t \tau_i^t,\; \forall i \in [K-1]
\end{align*}
where $\tau_i^t=z_i^t \I[z_i^t(\ww^t.\xx^t - \theta_i^t)\leq 0]$. 
\begin{algorithm}[h]
\caption{Regularized-PRank (Full Information Based Online Ordinal Regression)}
\label{algo:PRank}
\begin{algorithmic}
\STATE {\bf Input: } Training Dataset $\mathcal{S}$, $\eta_t$\;
\STATE {\bf Initialize} Set $t=1$, $\ww^1=\zero$, $\theta_1^{1}=\ldots=\theta_{K-1}^1=0$, $\theta_K^1=\infty$\;
\FOR{$i\leftarrow 1$ to $T$}
\STATE Get example $\xx^t$ and its label $y^t$
\STATE Set $z_i^{t}=+1,\;i\in\{1,\ldots,y^t-1\}$
\STATE Set $z_i^{t}=-1,\;i\in\{y^t,\ldots,K-1\}$
\STATE Initialize $\tau_i^t=0,\;i\in [K]$
\FOR{$i\in [K]$}
\IF{$z_i^t(\ww^t.\xx^t -\theta_i^t)\leq 0$}
\STATE $\tau_i^t = z_i^t$
\ENDIF
\ENDFOR
\STATE $\ww^{t+1} = (1-\eta_t\lambda)\ww^t + \eta_t\sum_{i=1}^{K-1}\tau_i^t \xx^t$\;
\STATE $\theta_i^{t+1}=(1-\eta_t\lambda)\theta_i^t -\eta_t\tau_i^t,\; i=1\ldots K-1$\;
\ENDFOR
\STATE {\bf Output}: $h(\xx)=\min_{i\in [K]}\big{\{}i\;:\;\ww^{T+1}.\xx-\theta_i^{T+1} <0 \big{\}}$
\end{algorithmic}
\end{algorithm} 

\subsection{PRIL (Interval Label Based Online Ordinal Regression)}
\label{sec:PRIL}
We discuss the algorithm for linear classifiers. Which means, $f(\xx)=\ww\cdot\xx$. We initialize with $\ww^0=\mathbf{0}$ and $\thetaa^0=\mathbf{0}$. Let $\ww^t,\thetaa^t$ be the estimates of the parameters at the beginning of trial $t$. Let at trial $t$, $\xx^t$ be the example observed, and $\hat{y}^t = \min_{i\in [K]} \{i: \ww^t\cdot\xx^t - \theta_i^t \leq 0\}$ be the label predicted. Now, we get the directional feedback $d_t=\mathbb{I}[\hat{y}^t<y^t]$. If $d_t=1$, then we get to know that $\hat{y}^t<y^t$ or $y^t\in \{\hat{y}^t+1,\ldots,K\}$. Thus, in case $d_t=1$, we set $y_l^t=\hat{y}^t+1$ and $y_r^t=K$. On the other hand, if $d_t=0$, then we find that $\hat{y}^t\geq y^t$ or $y^t\in \{1,\ldots, \hat{y}^t\}$. Thus, in case $d_t=0$, we set $y_l^t=1$ and $y_r^t=\hat{y}^t$. We define $z_i^t,\;i\in [K]$ as follows:
\begin{align*}
    z_i^{t}&=\begin{cases}
        \mathbb{I}\left[i\in\{1,\ldots,\hat{y}^t\}\right]+0\times \mathbb{I}\left[i\in \{\hat{y}^t+1,\ldots,K-1\}\right], & \text{if }\hat{y}^t<y^t\\
0\times \mathbb{I}[i\in\{1,\ldots,\hat{y}^t-1\}]-\mathbb{I}[i\in \{\hat{y}^t,\ldots,K-1\}], & \text{if }\hat{y}^t\geq y^t
\end{cases}
\end{align*}
At round $t$, the algorithm performs stochastic gradient descent on $L^{IH}_{reg}(f,\thetaa,\xx^t,y_l^t,y_r^t)$ to update the parameters. 
$\ww^{t+1}$ and $\thetaa^{t+1}$ are found as follows.
\begin{align*}
\ww^{t+1} &= \ww^t - \eta_t\nabla_{\ww} L^{IH}_{reg}(\ww^t.\xx^t,\thetaa^t,\hat{y}^t+1,K)\mathbb{I}[\hat{y}^t<y^t]- \eta_t\nabla_{\ww} L^{IH}_{reg}(\ww^t.\xx^t,\thetaa^t,1,\hat{y}^t)\mathbb{I}[\hat{y}^t\geq y^t]\\
&= (1-\eta_t\lambda)\ww^t + \eta_t\mathbb{I}[\hat{y}^t<y^t]\sum_{i=1}^{\hat{y}^t}\tau_i^t \xx^t - \eta_t\mathbb{I}[\hat{y}^t\geq y^t]\sum_{i=\hat{y}^t}^{K-1}\tau_i^t \xx^t\\
\theta^{t+1}_i &= \theta^t_i - \eta_t\frac{\partial  L^{H}_{reg}(\ww^t.\xx^t,\thetaa^t,\xx^t,y^t)}{\partial \theta_i}\\
&=\begin{cases}
    (1-\eta_t\lambda)\theta_i^t - \eta_t \tau_i^t\mathbb{I}[i\in \{1,\ldots,\hat{y}^t\}],\; \text{if }\hat{y}^t<y^t\\
    (1-\eta_t\lambda)\theta_i^t + \eta_t \tau_i^t\mathbb{I}[i\in \{\hat{y}^t,\ldots,K-1\}],\; \text{if }\hat{y}^t\geq y^t
    \end{cases}
\end{align*}
where $\tau_i^t=z_i^t \I[z_i^t(\ww^t.\xx^t - \theta_i^t)\leq 0].$ Complete details of the PRIL approach are given in Algorithm~\ref{algo:PRIL}.
\begin{algorithm}[h]
\caption{Regularized-PRIL (Interval Label Based Online Ordinal Regression)}
\label{algo:PRIL}
\begin{algorithmic}
\STATE {\bf Input: } Training Dataset $\mathcal{S}$, $\eta$\;
\STATE {\bf Initialize} Set $t=1$, $\ww^1=\zero$, $\theta_1^{1}=\ldots=\theta_{K-1}^1=0$, $\theta_K^1=\infty$\;
\FOR{$i\leftarrow 1$ to $T$}
\STATE Get example $\xx^t$
\STATE Find $\hat{y}^t$ as $\hat{y}^t = \min_{i\in [K]} \{i: \ww^t\cdot\xx^t - \theta_i^t \leq 0\}$
\STATE Observe the directional feedback $d_t=\mathbb{I}[\hat{y}^t<y^t]$
\IF{$d_t=1$}
\STATE Set $z_i^{t}=+1,\;i\in\{1,\ldots,\hat{y}^t\}$
\STATE Set $z_i^t=0,\;i\in \{\hat{y}^t+1,\ldots,K-1\}$
\ELSE
\STATE Set $z_i^{t}=0,\;i\in\{1,\ldots,\hat{y}^t-1\}$
\STATE Set $z_i^t=-1,\;i\in \{\hat{y}^t,\ldots,K-1\}$
\ENDIF
\STATE $\ww^{t+1} = (1-\eta_t\lambda)\ww^t + \eta_t\mathbb{I}[\hat{y}^t<y^t]\sum_{i=1}^{\hat{y}^t}\mathbb{I}[w^t.x^t < \theta^t_i] \xx^t -\eta_t\mathbb{I}[\hat{y}^t\geq y^t]\sum_{i=\hat{y}^t}^{K-1}\mathbb{I}[w^t.x^t > \theta^t_i] \xx^t$
\STATE $\theta^{t+1}_i =\begin{cases}
    (1-\eta_t\lambda)\theta_i^t - \eta_t \mathbb{I}[w^t.x^t < \theta^t_i]\mathbb{I}[i\in \{1,\ldots,\hat{y}^t\}],\; \text{if }\hat{y}^t<y^t\\
    (1-\eta_t\lambda)\theta_i^t + \eta_t \mathbb{I}[w^t.x^t >\theta^t_i]\mathbb{I}[i\in \{\hat{y}^t,\ldots,K-1\}],\; \text{if }\hat{y}^t\geq y^t
    \end{cases}$
\ENDFOR
\STATE {\bf Output}: $h(\xx)=\min_{i\in [K]}\big{\{}i\;:\;\ww^{T+1}.\xx-\theta_i^{T+1} <0 \big{\}}$
\end{algorithmic}
\end{algorithm}

\end{document}